\documentclass[11pt,twoside]{article}
\usepackage[hmargin=0.8in,vmargin=0.9in]{geometry}
\usepackage[pagewise]{lineno}

\usepackage{amsmath}
\usepackage{amssymb}
\usepackage{amsthm}
\usepackage{xcolor}
\usepackage{graphicx}

\usepackage{hyperref}
\usepackage{tikz-cd}

\usepackage{hyperref}

\title{Reproducing kernel Hilbert spaces in the mean field limit}
\author{ Christian Fiedler\thanks{Institute for Data Science in Mechanical Engineering, RWTH Aachen University, 52068 Aachen, Germany; {\tt \{christian.fiedler,trimpe\}@dsme.rwth-aachen.de}}, Michael Herty\thanks{Institute of Geometry and Practical Mathematics, RWTH Aachen University, 52056 Aachen,
		Germany; {\tt \{herty, rom, segala\}@igpm.rwth-aachen.de}}, Michael Rom\footnotemark[2], Chiara Segala\footnotemark[2], Sebastian Trimpe\footnotemark[1]  }
\date{\today}

\newcommand{\Permutations}{\mathcal{S}}
\newcommand{\Pb}{\mathcal{P}} 
\newcommand{\dkm}{d_{\text{KR}}} 
\newcommand{\McS}{{\text{McK}}} 
\newcommand{\tw}{\tilde{\omega}} 
\newcommand{\muhat}{\hat{\mu}}
\newcommand{\convmode}[1]{\overset{#1}{\longrightarrow}}
\newcommand{\weakstar}{weak$\ast$}
\newtheorem*{fact*}{Fact}

\theoremstyle{definition}
\newtheorem{defn}{Definition}[section]
\newtheorem{proposition}[defn]{Proposition}

\newtheorem{remark}[defn]{Remark}
\newtheorem{theorem}[defn]{Theorem}
\newtheorem{corollary}[defn]{Corollary}
\newtheorem{assumption}[defn]{Assumption}
\newtheorem{definition}[defn]{Definition}

\newtheorem*{theorem*}{Theorem}
\newtheorem*{definition*}{Definition}

\newcommand{\N}{\mathbb{N}} 
\newcommand{\R}{\mathbb{R}} 

\newcommand{\Rnn}{\mathbb{R}_{\geq 0}} 
\newcommand{\Rp}{\mathbb{R}_{>0}} 
\newcommand{\Np}{\mathbb{N}_+} 

\newcommand{\defm}[1]{\emph{#1}}

\begin{document}

\maketitle

\begin{abstract}
Kernel methods, being supported by a well-developed theory and coming with efficient algorithms, are among the most popular and successful machine learning techniques.
From a mathematical point of view, these methods rest on the concept of kernels and function spaces generated by kernels, so--called reproducing kernel Hilbert spaces.
Motivated by recent developments of learning approaches in the context of interacting particle systems, we investigate kernel methods acting on data with many measurement variables.
We show the rigorous mean field limit of kernels and provide a detailed analysis of the limiting reproducing kernel Hilbert space.
Furthermore, several examples of kernels, that allow a rigorous mean field limit, are presented.
\end{abstract}

\section{Introduction}

Interacting particle systems and related mean field  models have been recently used in a large variety of challenging domains, including biology, social systems and economics, see e.g. \cite{bellomo2017active,gibelli2019crowd,gibelli2020crowd2}.
These models originated in statistical mechanics, in particular, the statistical theory of rarefied gases \cite{pareschi2013interacting,cercignani2013mathematical}, but recent extensions to other areas of application like  biology or sociology lead to novel challenges, both on the mathematical as well as modeling level \cite{Cristiani:2014aa,Pareschi:2013aa}.  
Classical examples that are modeled by self-propelled particles include animals and robots, see e.g.~\cite{bellomo20review, MR2974186, MR2165531, MR3119732, MR2861587, MR2580958,Giselle}. Those particles interact according to nonlinear models encoding various social rules such as attraction, repulsion and alignment. A particular feature of such models are their rich dynamical phenomena, which include different types of emerging patterns like consensus, flocking, and milling \cite{MR2887663, MR2247927,cucker2007emergent,d2006self,motsch2014heterophilious}.

A particular difficulty arises from the complexity and heterogeneity of the systems under consideration in these new domains. In classical applications, like the statistical mechanics of gases, first-principle modeling approaches have been very successfully used. However, such principles are in general not directly available in the context of modeling of social behavior. Hence,  learning techniques in the field of interacting particle systems and related kinetic models \cite{bongini2017inferring} have been introduced as a substitute for possibly unknown modeling.  These learning approaches typically use trajectory data to approximate interaction rules, instead of deriving these from first principles.  This area has seen considerable activity in recent years, resulting in both algorithmic and theoretical advances, e.g.~\cite{lu2019nonparametric,lu2021learning,lu2021blearning}. The limit of  infinitely many particles leads to mean field equations for the evolution of the particle density.  They have been recently also used to tackle e.g.\ clustering problems in the advent of large-data (see e.g.~\cite{MR4160177}), deep neural networks \cite{mei2019mean,2639800120222271}  or large-scale optimization problems \cite{MR4329816,MR4121318,MR3597012,MR4222159,MR4444146,herty2017performance}. While in other instances emerging patterns could be  established by analyzing the limit of infinitely many particles \cite{MR2536247, MR2438213, MR2740099,  MR2425606, MR3194652,carrillo2014derivation}. We also refer to  \cite{carrillo2010particle,carrillo2014derivation,canizo2011well} for a rigorous treatment of the mean field limit of large scale interacting particle systems.

Motivated by these developments, we consider the mean field limit for kernel methods---a powerful class of machine learning methods. In particular, we consider the situation of learning methods operating on data with many inputs and utilize mean field theory to investigate the limit of infinitely many inputs. The interest in kernel methods in the large data limit is two-fold. First, those methods are commonly used and are a very popular and powerful machine learning tool, see e.g. \cite{scholkopf2002learning}, and \cite{williams2006gaussian, kanagawa2018gaussian} for results on Gaussian processes. Second, they are supported by a well--defined theory, making them amenable for rigorous analysis \cite{steinwart2008support,shawe2004kernel}.

To exemplify and for the sake of concreteness, we consider learning tasks involving functionals on the state  space $X$ of interacting particle systems with $M$ agents, see e.g. \cite{MR2247927}.  
Various properties of such dynamics, like mean, variance or other statistical properties of the particle system, can then be described by a functional $f_M: X^M \rightarrow\R$ on the state space $X^M$ of the $M-$interacting particle system.  Hence,  $f_M$ is an  observable of the particle system and hence be subject to measurements, respectively data. It is  a natural question to ask whether an approximation $\hat{f}_M$ to $f_M$ can be learned by machine learning tools. Kernel-based methods proceed in this context as follows: they generate an approximation  $\hat{f}_M$ using a weighted sum 
\begin{equation}\label{1}
	\hat{f}_M = \sum_{n=1}^N \alpha_n k_M(\cdot, \vec x_n),
\end{equation}
where $\alpha_1,\ldots,\alpha_N\in\R$ are coefficients, $\vec x_1,\ldots,\vec x_N\in[0,1]^M$ are data points and $k_M: [0,1]^M\times[0,1]^M\rightarrow\R$ is a kernel function, cf.\ e.g.\ \cite{scholkopf2001generalized} and below in Section \ref{section.examples} for further examples.   Note that the approximation $\hat{f}_M$ (and the kernel $k_M$) depends on the size of the particle system $M$ through the measurements. Clearly, the dimension of $x_i \in [0,1]$ is not restrictive and any higher dimensional state space  is possible.  The set of all  functionals represented by the series \eqref{1} forms a Hilbert space $H_M$ (see Section \ref{section.prelim} for a precise definition of reproducing kernel Hilbert spaces (RKHS)). Hence, the question is closely related to the problem of describing the Hilbert space generated by the kernel $k_M.$  
\par 
We are now interested in the limit of those approximations in the case $M\to \infty$. Convergence of functionals $\hat{f}_M$ in the case $M\to \infty$ can be established provided that the $\hat{f}_M$ are symmetric, see  \cite{cardaliaguet2010}. In this paper, we discuss whether also $k_M$ has a similar limit $k$ and whether it can be used to model functionals on the mean field level. { Existence and properties of the limiting kernel $k,$ that is a mapping $\Pb(X)\times\Pb(X) \rightarrow \R$, is a first main result, that is summarized in Theorem 3.2. Given the limit of those kernels on the space of  probability measures $\mathcal{P}(X)$ allows then to  
}
establish that in fact $\hat{f}$ is expressed through a kernel $k$. 
{ This result is given in Theorem 4.4 below and 
	may be } 
represented by 
\begin{equation}\label{2}
	\hat{f} = \sum_{n=1}^N \alpha_n k(\cdot,\mu_n),
\end{equation}
where now $\mu_n \in \mathcal{P}(X)$. { Furthermore, 
	Theorem 4.4 allows to investigate if the previous functions also form   a reproducing kernel Hilbert space $H_k$. 
}
The  diagram of Figure \ref{fig.mfl_kernel_diagram1} summarizes the obtained relations. 

\begin{figure}\center
	\tikzset{every picture/.style={line width=0.75pt}} 
	\begin{tikzpicture}[x=0.75pt,y=0.75pt,yscale=-1,xscale=1]
		
		\draw    (46.67,25.71) -- (196.54,27.12) ;
		\draw [shift={(198.54,27.14)}, rotate = 180.54] [color={rgb, 255:red, 0; green, 0; blue, 0 }  ][line width=0.75]    (10.93,-3.29) .. controls (6.95,-1.4) and (3.31,-0.3) .. (0,0) .. controls (3.31,0.3) and (6.95,1.4) .. (10.93,3.29)   ;
		\draw    (211.27,50.36) -- (211.27,102.29) ;
		\draw [shift={(211.27,104.29)}, rotate = 270] [color={rgb, 255:red, 0; green, 0; blue, 0 }  ][line width=0.75]    (10.93,-3.29) .. controls (6.95,-1.4) and (3.31,-0.3) .. (0,0) .. controls (3.31,0.3) and (6.95,1.4) .. (10.93,3.29)   ;
		\draw    (23.16,50.36) -- (23.16,102.29) ;
		\draw [shift={(23.16,104.29)}, rotate = 270] [color={rgb, 255:red, 0; green, 0; blue, 0 }  ][line width=0.75]    (10.93,-3.29) .. controls (6.95,-1.4) and (3.31,-0.3) .. (0,0) .. controls (3.31,0.3) and (6.95,1.4) .. (10.93,3.29)   ;
		\draw    (44.71,127.14) -- (195.5,127.99) ;
		\draw [shift={(197.5,128)}, rotate = 180.32] [color={rgb, 255:red, 0; green, 0; blue, 0 }  ][line width=0.75]    (10.93,-3.29) .. controls (6.95,-1.4) and (3.31,-0.3) .. (0,0) .. controls (3.31,0.3) and (6.95,1.4) .. (10.93,3.29)   ;
		
		\draw (22.67,21.43) node    {$k_M$};
		\draw (117.7,38.57) node  [font=\scriptsize,color={rgb, 255:red, 208; green, 2; blue, 27 }  ,opacity=1 ]  {$M\ \rightarrow \infty $};
		\draw (214.21,24.64) node    {$k$};
		\draw (22.67,125.71) node    {$H_{M}$};
		\draw (212.74,128.57) node    {$H_{k}$};
		\draw (117.7,10) node  [font=\scriptsize,color={rgb, 255:red, 208; green, 2; blue, 27 }  ,opacity=1 ]  {MFL of $k_M$};
		\draw (115.74,140) node  [font=\scriptsize,color={rgb, 255:red, 208; green, 2; blue, 27 }  ,opacity=1 ]  {$M\ \rightarrow \infty $};
		\draw (115.74,111.43) node  [font=\scriptsize,color={rgb, 255:red, 208; green, 2; blue, 27 }  ,opacity=1 ]  {MFL of $f_{M}$};
	\end{tikzpicture}
	\caption{Commutative diagram summarizing the relation between mean field limit (MFL) of a sequence of kernels $(k_M)_M$ and their corresponding reproducing kernel Hilbert spaces. Here, $f_M$ denotes an element of the space $H_M$, and $k,H_k$ indicate the MFL of $(k_M)_M$ and $(H_M)_M$, respectively. The mean field limits are given in Theorem 3.2 and Theorem 4.4., respectively. }
	\label{fig.mfl_kernel_diagram1}
\end{figure}

The manuscript is organized as follows.  In Section \ref{section.prelim}, we collect some background material on mean field limits as well as kernels and their reproducing kernel Hilbert spaces.  Section \ref{section.mfl_kernels} presents the appropriate conditions on sequences of kernels to allow a rigorous mean field limit, which is proven in Theorem \ref{thm.mfl_kernel}. Then,  the induced RKHS is investigated in Section \ref{section.mfl_rkhs}. Finally, two large classes of such kernel sequences are presented and analyzed in Section \ref{section.examples}, and Section \ref{section.conclusion} contains a summary and an outlook.

\section{Preliminary discussion and notation} \label{section.prelim}
Unless noted otherwise,  $(X,d_X)$ denotes a compact metric space.
Let $\Pb(X)$ be the set of Borel probability measures on $X$, which we endow with the topology of \weakstar\ convergence, i.e.,
$(\mu_n)_n\subseteq\Pb(X)$ converges to $\mu\in\Pb(X)$ iff for all continuous $\phi: X\rightarrow\R$ we have
\begin{equation*}
	\lim_{n\rightarrow\infty} \int_X \phi(x)\mathrm{d}\mu_n(x) = \int_X \phi(x)\mathrm{d}\mu(x).
\end{equation*}
It is well-known that $\Pb(X)$ is compact and can be metrized by the Kantorowich-Rubinstein distance $\dkm$, defined by
\begin{equation*}
	\dkm(\mu_1,\mu_2)= \sup\left\{ \int_X \phi(x) \mathrm{d}(\mu_1-\mu_2)(x) \mid \phi: X \rightarrow \R \text{ is 1-Lipschitz } \right\}.
\end{equation*}
We also define $\dkm^2: \Pb(X)\times\Pb(X)\rightarrow\Rnn$ by 
\begin{equation*}
	\dkm^2((\mu_1,\mu_1^\prime),(\mu_2,\mu_2^\prime))=\dkm(\mu_1,\mu_2)+\dkm(\mu_1^\prime,\mu_2^\prime),
\end{equation*}
and note that $(\Pb(X)\times\Pb(X),\dkm^2)$ is  a compact metric space. For $M\in\Np$ and $\vec x\in X^M$, denote the $i$-th component of $\vec x$ by $x_i$, and define the \defm{empirical measure} for $\vec x$ by
\begin{equation*}
	\muhat[\vec x] = \frac{1}{M}\sum_{i=1}^M \delta_{x_i},
\end{equation*}
where $\delta_x$ denotes the Dirac measure centered at $x\in X$. The  empirical measures are dense in $\Pb(X)$ under the given metric. 
{Furthermore, denote the set of permutations on $\{1,\ldots, M\}$ by $\Permutations_M$. }
\par 
In the context of this article, a {modulus of continuity} is a function $\omega:\Rnn\rightarrow\Rnn$ that is continuous, non decreasing and with $\omega(0)=0$.
Later we use that for every $R\in\Rp$ and every modulus of continuity $\omega$, we can find a concave modulus of continuity $\tilde \omega: [0,R] \rightarrow \Rnn$ such that $\omega(r)\leq\tilde\omega(r)$ for all $r\in[0,R]$. 
{We define $X^M = X \times X \dots \times X$ the metric space of $M$ copies of $(X,d_X).$ }

We  recall a result on symmetric functions of many variables, which motivates our later developments.
\begin{assumption} \label{assump.cardaliaguet2010}
	Let $f_M: X^M \rightarrow \R$, $M\in\Np$, such that
	\begin{enumerate}
		\item \emph{(Symmetry in $\vec x$)} For all $M\in\Np$, $\vec{x}\in X^M$ and permutations $\sigma\in\Permutations_M$, we have
		\begin{equation*}
			f_M(\sigma\vec{x})
			:=f(x_{\sigma(1)},\ldots,x_{\sigma(M)})
			=f(\vec{x})
		\end{equation*}
		\item \emph{(Uniform boundedness)} There exists $C_f\in\Rnn$ such that
		\begin{equation*}
			\forall M\in\Np,\vec{x} \in X^M: |f_M(\vec{x})|\leq C_f
		\end{equation*}
		\item \emph{(Uniform continuity)} There exists a modulus of continuity $\omega_f:\Rnn\rightarrow\Rnn$ such that for all $M\in\Np$, $\vec{x}_1,\vec{x}_2\in X^M$
		\begin{equation*}
			|f_M(\vec{x}_1)-f_M(\vec{x}_2)|
			\leq \omega_f\left( \dkm(\muhat[\vec{x}_1],\muhat[\vec{x}_2]) \right)
		\end{equation*}
	\end{enumerate}
\end{assumption}
Note that in Assumption \ref{assump.cardaliaguet2010}, the symmetry in $\vec x$ for $f_M$ is actually implied by the uniform continuity, cf. \cite[Remark~1.3]{carmona2018probabilistic}.
Furthermore, this latter property also implies continuity with respect to the product metric on $X^M$. The following result will also be used to establish the mean field convergence and is repeated here for convenience \cite[Theorem~2.1]{cardaliaguet2010}.
\begin{theorem} \label{thm.mfl_funcs}
	Under Assumption \ref{assump.cardaliaguet2010}, there exists a subsequence $(f_{M_\ell})_\ell$ and some $f: \Pb(X)\rightarrow\Rnn$, such that
	\begin{equation*}
		\lim_{\ell\rightarrow\infty} \sup_{\vec x\in X^{M_\ell}} |f_{M_\ell}(\vec x) - f(\muhat[\vec x])| = 0.
	\end{equation*}
	Furthermore, $f$ is continuous as function on $\Pb(X)$ and (uniformly) bounded by $C_f$.
\end{theorem}

\subsection{Reproducing kernel Hilbert spaces} \label{section.rkhs}
Concepts and results on reproducing kernel Hilbert spaces are recalled for the convenience of the reader. This presentation follows closely \cite[Chapter~4]{steinwart2008support}, where also further aspects are detailed. 

{ 
	\begin{definition} Let $\mathcal{X} \not = \emptyset$ be an arbitrary set and $H \subseteq \R^\mathcal{X}$ a real Hilbert space of functions. Then, the function $k: \mathcal{X} \times \mathcal{X} \rightarrow \R$ is called a kernel, if  there exists a real Hilbert space $\mathcal{H}$ and a map $\Phi: \mathcal{X} \rightarrow \mathcal{H}$ 
		\begin{equation*}
			k(x,x') = \langle \Phi(x'), \Phi(x) \rangle_\mathcal{H} \forall x,x'\in \mathcal{X}.
		\end{equation*}
		The map $\Phi$ is called feature map and $\mathcal{H}$ is the  \defm{feature space}.  
		\par 
		The mapping $k$ is called a \defm{reproducing kernel} (for $H$) if 
		$$k(\cdot,x)\in H \; \forall x\in \mathcal{X}$$ and if $$f(x)=\langle f, k(\cdot, x)\rangle_H \; \forall f\in H, x \in \mathcal{X}.$$
		If a kernel $k$ has the previous property, than $k$ is said to have  the \defm{reproducing property}.
		\par 
		$H$ is called a \defm{reproducing kernel Hilbert space} if all evaluation functionals are continuous, i.e., if 
		for all $x \in \mathcal{X}$ the functionals $$\delta_x: H \rightarrow \R, \; \delta_x(f)=f(x)$$ are bounded.
		\par
		Finally, we say that $k$ is positive definite\footnote{In the literature this is sometimes called positive semi-definiteness.}  if for all $N\in\mathbb{N}$, $x_1,\ldots,x_N\in\mathcal{X}$ and $\alpha_1,\ldots,\alpha_N\in\mathbb{R}$ we have
		\begin{equation*}
			\sum_{i,j=1}^N \alpha_i \alpha_j k(x_j,x_i) \geq 0.
		\end{equation*}
	\end{definition} 
}

For convenience, we now recall some well-known properties on RKHS and their kernels.
\begin{enumerate}
	\item If $H$ has a reproducing kernel, then it is a RKHS.
	\item Every RKHS has a unique reproducing kernel.
	\item If $k$ is a reproducing kernel for $H$, then $k$ is a kernel and $\Phi_k: \mathcal{X}\rightarrow H$, $\Phi_k(x)=k(\cdot,x)$ is a feature map (called \defm{canonical feature map}) and $H$ is a feature space for $k$.
	\item $k$ is a kernel if and only if it is symmetric and positive definite.
	\item Every kernel $k$ has a unique RKHS for which it is a reproducing kernel. We denote this RKHS by $H_k$, its associated scalar product by $\langle \cdot, \cdot,\rangle_k$ (or just $\langle \cdot, \cdot\rangle$ if $k$ is clear from context) and the induced norm by $\|\cdot\|_k$.
	\item The pre-Hilbert space
	\begin{equation*}
		\begin{split}
			H_\text{pre} & = \text{span} \{ k(\cdot,x) \mid x \in \mathcal{X} \}
			\\
			& = \left\{ \sum_{n=1}^N \alpha_n k(\cdot, x_n) \mid N\in\mathbb{N},\: \alpha_n \in\mathbb{R}, x_n \in \mathcal{X}, n=1,\ldots,N \right\}
		\end{split}
	\end{equation*}
	with the inner product
	\begin{equation*}
		\left\langle \sum_{n=1}^N \alpha_n k(\cdot,x_n), \sum_{m=1}^M \beta_m k(\cdot,y_m) \right\rangle
		= \sum_{n=1}^N \sum_{m=1}^M \alpha_n \beta_m k(y_m, x_n)
	\end{equation*}
	is dense in the (unique) RKHS $H_k$ for kernel $k$.
\end{enumerate}
Additionally, to every kernel $k: X \times X \rightarrow \R$ we can associated the \defm{kernel metric} induced by $k$,
\begin{equation*}
	d_k: X \times X \rightarrow \Rnn, \: d_k(x,x')=\|\Phi_{k}(x)-\Phi_{k}(x')\|_{k} = \sqrt{k(x,x) - 2k(x,x') + k(x',x')}.
\end{equation*}
The kernel metric $d_k$ is always a pseudometric on $X$, even if $X$ { has} no structure by itself, and a metric on $X$ if $\Phi_k$ is injective.

Furthermore, in Section \ref{section.examples} we need the notion of \defm{kernel mean embeddings} (KME) of distributions, see \cite{muandet2017kernel}.
Let $X$ be a compact metric space and $k: X \times X \rightarrow \Rnn$ a continuous and bounded kernel on $X$.
Then $x \mapsto k(\cdot,x)$ is measurable and Bochner integrable for every Borel probability measure $\mu\in\Pb(X)$.
{
	\begin{definition}
		We define the kernel mean embedding of $\mu$ into $H_k$ by
		\begin{equation*}
			f^k_\mu = \int k(\cdot, x) \mathrm{d}\mu(x).
		\end{equation*}
		If the map $\Pb(X)\rightarrow H_k$, $\mu \mapsto f_\mu^k$ is injective, we call $k$ \defm{characteristic}.
	\end{definition}
}

\section{The mean field limit of kernels} \label{section.mfl_kernels}
In this section, we investigate the mean field limit of sequences of kernels. In order to show the dependence of the kernel on the  dimension $M$, we use an upper index. Let $X$ be as in the previous section and consider now a sequence 

$$k^{[M]}: X^M \times X^M \rightarrow \R,  \; M\in\N_+,$$  of kernels on input space $X^M$ where we impose the following assumptions.

\begin{assumption} \label{assump.mfl_kernel_cond}
	\begin{enumerate}
		\item \emph{(Symmetry in $\vec x$)} For all $M\in\N_+$, $\vec{x},\vec{x}^\prime\in X^M$ and permutations $\sigma\in \Permutations_M$ we have
		\begin{equation*}
			k^{[M]}(\sigma\vec{x},\vec{x}^\prime)
			:=k^{[M]}((x_{\sigma(1)},\ldots,x_{\sigma(M)}),\vec{x}^\prime)
			=k^{[M]}(\vec{x},\vec{x}^\prime)
		\end{equation*}
		\item \emph{(Uniform boundedness)} There exists $C_k\in\Rnn$ such that
		\begin{equation*}
			\forall M\in\N_+,\vec{x},\vec{x}^\prime \in X^M: |k^{[M]}(\vec{x},\vec{x}^\prime)|\leq C_k
		\end{equation*}
		\item \emph{(Uniform continuity)} There exists a modulus of continuity $\omega_k:\Rnn\rightarrow\Rnn$ such that for all $M\in\N_+$, $\vec{x}_1,\vec{x}_1^\prime,\vec{x}_2,\vec{x}_2'\in X^M$
		\begin{equation*}
			|k^{[M]}(\vec{x}_1,\vec{x}_1^\prime)-k^{[M]}(\vec{x}_2,\vec{x}_2^\prime)|
			\leq \omega_k\left( \dkm^2\left[ (\muhat[\vec{x}_1],\muhat[\vec{x}_1^\prime),  (\muhat[\vec{x}_2],\muhat[\vec{x}_2^\prime)\right] \right)
		\end{equation*}
	\end{enumerate}
\end{assumption}
%
%
The next theorem extends the proof in  \cite[Theorem~2.1]{cardaliaguet2010} and shows that, if a sequence of kernels fulfills Assumption \ref{assump.mfl_kernel_cond}, then there exists the mean field limit, which is again a {\em kernel.}
\begin{theorem} \label{thm.mfl_kernel}
	Under Assumption \ref{assump.mfl_kernel_cond}, there exists a subsequence $(k^{[M_\ell]})_{\ell}$ and a continuous, bounded kernel $k: \Pb(X) \times \Pb(X) \rightarrow \R$ such that
	\begin{equation} \label{eq.conv_kernels}
		\lim_{ {\ell} \rightarrow \infty} \sup_{\vec{x},\vec{x}^\prime\in X^{M_{\ell}}} 
		|k^{[M_{\ell}]}(\vec{x},\vec{x}^\prime)-k(\muhat[\vec{x}],\muhat[\vec{x}^\prime])| = 0.
	\end{equation}
\end{theorem}
{ Note that the theorem states the existence of a limiting kernel $k$, independent of $M$. The mapping $k:\Pb(X) \times \Pb(X) \rightarrow \R$ is called a kernel over the probability space $\Pb(X).$ It fulfills the following properties that will be established in the proof below: 
	\begin{itemize}
		\item $k$ is symmetric and positive definite on $\Pb(X) \times \Pb(X)$
		\item $k$ is bounded on $\Pb(X) \times \Pb(X)$
	\end{itemize} 
	
}

The first part of the proof is based on the same arguments as in  \cite[Theorem~2.1]{cardaliaguet2010} and repeated only for convenience.
\begin{proof}
	{ We construct a sequence of uniformly bounded and equi--continuous kernels $k^{[M]}_\McS$ for $M \in \mathbb{N}_+.$ 
		Its limit will be the desired kernel $k.$ 
	}
	\par 
	\textbf{Step 1.}
	{In the first step we define $k^{[M]}_\McS$ and show that it is bounded  on $\Pb(X)\times\Pb(X)$ and coincides with the kernel $k^{[M]}$ on $X^M \times X^M$. }  
	Since $\Pb(X)$ is compact, it has a finite diameter $D_{\Pb(X)}\in\Rnn$. Let $\tw_k:[0,2D_{\Pb(X)}]\rightarrow\Rnn$ be a modulus of continuity, that is a pointwise upper bound to $\omega_k$.
	For all $M\in\Np$, define now the McKean extension $k^{[M]}_\McS: \Pb(X)\times\Pb(X)\rightarrow \R$ by
	\begin{align*}
		k^{[M]}_\McS(\mu,\mu^\prime):= \inf_{\vec{x},\vec{x}^\prime\in X^M} k^{[M]}(\vec{x},\vec{x}^\prime) + \tw_k\left( \dkm^2 \left[ (\muhat[\vec{x}], \muhat[\vec{x}^\prime]), (\mu,\mu^\prime) \right] \right) {.}
	\end{align*}
	
	Note that for all $M\in\N_+$, $k^{[M]}_\McS$ is well-defined.
	For this, we show that 
	\\
	$\dkm^2 \left[ (\muhat[\vec{x}], \muhat[\vec{x}^\prime]), (\mu,\mu^\prime) \right] $ belongs to the domain of $\tw_k$. This holds true, since 
	\begin{equation*}
		\dkm^2 \left[ (\muhat[\vec{x}], \muhat[\vec{x}^\prime]), (\mu,\mu^\prime) \right] 
		\leq \dkm(\muhat[\vec{x}], \mu) +  \dkm(\muhat[\vec{x}^\prime], \mu^\prime) \leq 2D_{\Pb(X)}.
	\end{equation*}
	Second, we show that  $k^{[M]}_\McS(\mu,\mu^\prime)$ {is bounded.}
	Since $X$ and hence $\Pb(X)$ are non-empty, we have $k^{[M]}_\McS(\mu,\mu^\prime)<\infty$.
	The uniform continuity assumption on $k^{[M]}$ implies that all kernels are continuous as functions on $X^{2M}$ and therefore (recall that $\tw_k\geq 0$)
	\begin{equation*}
		k^{[M]}_\McS(\mu,\mu^\prime) \geq \inf_{\vec{x},\vec{x}^\prime\in X^M} k^{[M]}(\vec{x},\vec{x}^\prime) > -\infty
	\end{equation*}
	by compactness of $X^M \times X^M$.
	
	Furthermore, observe that for all $M\in\N_+$ and $\vec{x},\vec{x}^\prime\in X^M$, we have
	\begin{equation} \label{eq.mcshane_atomic}
		k^{[M]}_\McS(\muhat[\vec{x}],\muhat[\vec{x}^\prime])=k^{[M]}(\vec{x},\vec{x}^\prime) .
	\end{equation}
	For arbitrary $\vec x,\vec x'\in X^M$ it holds by  construction
	\begin{align*}
		k^{[M]}_\McS(\muhat[\vec{x}],\muhat[\vec{x}^\prime])
		& \leq k^{[M]}(\vec{x},\vec{x}^\prime) + \tw_k\left( \dkm^2 \left[ (\muhat[\vec{x}], \muhat[\vec{x}^\prime]), (\muhat[\vec{x}], \muhat[\vec{x}^\prime]) \right] \right) \\
		& = k^{[M]}(\vec{x},\vec{x}^\prime).
	\end{align*}
	Let additionally $\vec{x}_1,\vec{x}_1'\in X^M$ be arbitrary, then we obtain
	\begin{align*}
		& k^{[M]}(\vec{x}_1,\vec{x}_1^\prime) 
		+ \tw_k\left( \dkm^2 \left[ 
		(\muhat[\vec{x}_1], \muhat[\vec{x}_1^\prime]),  
		(\muhat[\vec{x}], \muhat[\vec{x}^\prime]) 
		\right] \right) \\
		& \geq k^{[M]}(\vec{x},\vec{x}^\prime) 
		- |k^{[M]}(\vec{x}_1,\vec{x}_1^\prime) - k^{[M]}(\vec{x},\vec{x}^\prime)| 
		\\
		& \quad + \tw_k\left( \dkm^2 \left[ 
		(\muhat[\vec{x}_1], \muhat[\vec{x}_1^\prime]),  
		(\muhat[\vec{x}], \muhat[\vec{x}^\prime]) 
		\right] \right) \\
		& \geq k^{[M]}(\vec{x},\vec{x}^\prime) 
		- \omega_k\left( \dkm^2 \left[ 
		(\muhat[\vec{x}_1], \muhat[\vec{x}_1^\prime]),  
		(\muhat[\vec{x}], \muhat[\vec{x}^\prime]) 
		\right] \right) 
		\\
		& \quad + \tw_k\left( \dkm^2 \left[ 
		(\muhat[\vec{x}_1], \muhat[\vec{x}_1^\prime]),  
		(\muhat[\vec{x}], \muhat[\vec{x}^\prime]) 
		\right] \right) \\
		& \geq k^{[M]}(\vec{x},\vec{x}^\prime) 
		- \tw_k\left( \dkm^2 \left[ 
		(\muhat[\vec{x}_1], \muhat[\vec{x}_1^\prime]),  
		(\muhat[\vec{x}], \muhat[\vec{x}^\prime])
		\right] \right) 
		\\
		& \quad + \tw_k\left( \dkm^2 \left[ 
		(\muhat[\vec{x}_1], \muhat[\vec{x}_1^\prime]),  
		(\muhat[\vec{x}], \muhat[\vec{x}^\prime]) 
		\right] \right) \\
		& =  k^{[M]}(\vec{x},\vec{x}^\prime),
	\end{align*}
	where we used the uniform continuity of $k^{[M]}$ in the second inequality and the definition of $\tw_k$ (together with $\dkm^2 \left[ (\muhat[\vec{x}_1], \muhat[\vec{x}_1^\prime]),  (\muhat[\vec{x}], \muhat[\vec{x}^\prime]) \right] \leq 2D_{\Pb{X}}$) in the third inequality. This implies that
	$k^{[M]}_\McS(\muhat[\vec{x}],\muhat[\vec{x}^\prime]) \geq  k^{[M]}(\vec{x},\vec{x}^\prime)$.
	
	\textbf{Step 2} We now show equi-boundedness of $(k_\McS^{[M]})_M$. 
	Let $M\in\N_+$ and $\mu,\mu^\prime\in\Pb(X)$ be arbitrary, then
	\begin{align*}
		|k^{[M]}_\McS(\mu,\mu^\prime)| 
		& = \left| \inf_{\vec{x},\vec{x}^\prime\in X^M} k^{[M]}(\vec{x},\vec{x}^\prime) + \tw_k\left( \dkm^2 \left[ (\muhat[\vec{x}], \muhat[\vec{x}^\prime]), (\mu,\mu^\prime) \right] \right) \right| \\
		& \leq \inf_{\vec{x},\vec{x}^\prime\in X^M} 
		| k^{[M]}(\vec{x},\vec{x}^\prime)| + \left|\tw_k\left( \dkm^2 \left[ (\muhat[\vec{x}], \muhat[\vec{x}^\prime]), (\mu,\mu^\prime) \right] \right)\right| \\
		& \leq  C_k + \tw_k(2D_{\Pb(X)}) =: \tilde{C}_k,
	\end{align*}
	where we used the uniform boundedness of $k^{[M]}$ and the compactness  of $\Pb(X)$.
	
	\textbf{Step 3} Next, we show that $\tw_k$ is a modulus of continuity, i.e., 
	for all $M\in\N_+$, $\mu_1,\mu_1^\prime,\mu_2,\mu_2^\prime\in\Pb(X)$ we have
	\begin{equation*}
		|k^{[M]}_\McS(\mu_1,\mu_1^\prime) - k^{[M]}_\McS(\mu_2,\mu_2^\prime)| \leq \tw_k(\dkm^2[(\mu_1,\mu_1^\prime),(\mu_2,\mu_2^\prime)]).
	\end{equation*}
	To establish this, let $M\in\N_+$, $\mu_1,\mu_1^\prime,\mu_2,\mu_2^\prime\in\Pb(X)$ and $\epsilon >0$ be arbitrary.
	Now, let $(\vec{x}_2,\vec{x}_2^\prime)\in X^{2M}$ be  $\epsilon$-close, i.e.,
	\begin{equation*}
		k^{[M]}(\vec{x}_2,\vec{x}_2^\prime) + \tw_k\left( \dkm^2 \left[ (\muhat[\vec{x}_2], \muhat[\vec{x}_2^\prime]), (\mu_2,\mu_2^\prime) \right] \right) \leq k^{[M]}_\McS(\mu_2,\mu_2^\prime) + \epsilon.
	\end{equation*}
	Then, it holds
	\begin{align*}
		k^{[M]}_\McS(\mu_1,\mu_1^\prime) &  \leq k^{[M]}(\vec{x}_2,\vec{x}_2^\prime) 
		+ \tw_k\left( \dkm^2 \left[ 
		(\muhat[\vec{x}_2], \muhat[\vec{x}_2^\prime]), (\mu_1,\mu_1^\prime)
		\right] \right) \\
		& = k^{[M]}(\vec{x}_2,\vec{x}_2^\prime) 
		+ \tw_k\left( \dkm^2 \left[ 
		(\muhat[\vec{x}_2], \muhat[\vec{x}_2^\prime]), (\mu_2,\mu_2^\prime)
		\right] \right) \\
		& \quad -  \tw_k\left( \dkm^2 \left[ 
		(\muhat[\vec{x}_2], \muhat[\vec{x}_2^\prime]), (\mu_2,\mu_2^\prime)
		\right] \right)
		+ \tw_k\left( \dkm^2 \left[ 
		(\muhat[\vec{x}_2], \muhat[\vec{x}_2^\prime]), (\mu_1,\mu_1^\prime)
		\right] \right) \\
		& \leq k^{[M]}_\McS(\mu_2,\mu_2^\prime) + \epsilon
		-  \tw_k\left( \dkm^2 \left[ 
		(\muhat[\vec{x}_2], \muhat[\vec{x}_2^\prime]), (\mu_2,\mu_2^\prime)
		\right] \right)
		\\
		& \quad + \tw_k\left( \dkm^2 \left[ 
		(\muhat[\vec{x}_2], \muhat[\vec{x}_2^\prime]), (\mu_1,\mu_1^\prime)
		\right] \right) \\
		& \leq k^{[M]}_\McS(\mu_2,\mu_2^\prime) 
		+ \epsilon
		-  \tw_k\left( \dkm^2 \left[ 
		(\muhat[\vec{x}_2], \muhat[\vec{x}_2^\prime]), (\mu_2,\mu_2^\prime)
		\right] \right) \\
		& \quad + \tw_k\left( \dkm^2 \left[ 
		(\muhat[\vec{x}_2], \muhat[\vec{x}_2^\prime]), (\mu_2,\mu_2^\prime)
		\right] 
		+
		\dkm^2\left[
		(\mu_2,\mu_2^\prime), (\mu_1,\mu_1^\prime)
		\right]
		\right) \\
		& \leq k^{[M]}_\McS(\mu_2,\mu_2^\prime) 
		+ \epsilon
		+  \tw_k\left( \dkm^2\left[(\mu_2,\mu_2^\prime), (\mu_1,\mu_1^\prime) \right] \right),
	\end{align*}
	where we used the definition of $k^{[M]}_\McS(\mu_1,\mu_1^\prime)$ in the first inequality,
	the choice of $(\vec{x}_2,\vec{x}_2^\prime)$ in the second inequality,
	the triangle inequality for $\dkm$ together with the monotonicity of $\tw_k$ in the third inequality
	and finally the subadditivity. 
	Repeating these steps with the roles interchanged shows that
	\begin{equation*}
		|k^{[M]}_\McS(\mu_1,\mu_1^\prime) + k^{[M]}_\McS(\mu_2,\mu_2^\prime)| \leq \tw_k(\dkm^2\left[(\mu_1,\mu_1^\prime),(\mu_2,\mu_2^\prime)\right]) + \epsilon
	\end{equation*}
	and since $\epsilon>0$ was arbitrary and $\tw_k$ does not depend on $M$, the claim follows.
	
	\textbf{Step 4} 
	Summarizing, $(k^{[M]}_\McS)_{M\in\N_+} \subseteq C^0(\Pb(X)\times\Pb(X),\R)$ is a uniformly bounded, equi-continuous sequence. The Arzela-Ascoli theorem  guarantees existence of $k\in C^0(\Pb(X)\times\Pb(X),\R)$ and an unbounded sequence $(M_{\ell})_{ {\ell} \in\N_+}$ such that
	\begin{equation*}
		\lim_{ {\ell} \rightarrow \infty} \sup_{\mu,\mu^\prime \in \Pb(X)} |k^{[M_{\ell}]}_\McS(\mu,\mu^\prime)-k(\mu,\mu^\prime)| = 0.
	\end{equation*}
	This implies also  \eqref{eq.conv_kernels}. 
	To prove this, note that for all ${\ell}\in\N_+$ and $\vec{x}\in X^{M_{\ell}}$ we have $\muhat[\vec{x}]\in \Pb(X)$, and hence
	\begin{align*}
		& \lim_{{\ell}\rightarrow \infty} \sup_{\vec{x},\vec{x}^\prime\in X^{M_{\ell}}} 
		|k^{[M_{\ell}]}(\vec{x},\vec{x}^\prime)-k(\muhat[\vec{x}],\muhat[\vec{x}^\prime])|
		\\
		& = \lim_{{\ell}\rightarrow \infty} \sup_{\vec{x},\vec{x}^\prime\in X^{M_{\ell}}} 
		|k^{[M_{\ell}]}_\McS(\muhat[\vec{x}],\muhat[\vec{x}^\prime])-k(\muhat[\vec{x}],\muhat[\vec{x}^\prime])| \\
		& \leq  \lim_{{\ell}\rightarrow \infty} \sup_{\mu,\mu'\in\Pb(X)} |k^{[M_{\ell}]}(\mu,\mu')-k(\mu,\mu')|\\
		& = 0,
	\end{align*}
	where we used \eqref{eq.mcshane_atomic} in the first equality.
	
	\textbf{Step 5} Next, we show that for all $\mu_1,\mu_2\in\Pb(X)$, $|k(\mu_1,\mu_2)|\leq C_k$, i.e., {the function} $k$ is bounded.
	For this, let $\mu_1,\mu_2\in\Pb(X)$ and $\epsilon>0$ be arbitary. 
	Choose $n\in\N_+$ such that
	\begin{equation*}
		\|k^{[M_n]}_\McS - k\|_\infty = \sup_{\mu,\mu^\prime\in\Pb(X)}|k^{[M_n]}_\McS(\mu,\mu^\prime)-k(\mu,\mu^\prime)| \leq \epsilon.
	\end{equation*}
	We then have
	\begin{equation*}
		|k(\mu_1,\mu_2)| \leq |k(\mu_1,\mu_2) - k^{[M_n]}_\McS(\mu_1,\mu_2)| + |k^{[M_n]}_\McS(\mu_1,\mu_2)| 
		\leq \epsilon + C_k,
	\end{equation*}
	due to the uniform boundedness of $k^{[M_n]}$. Since $\epsilon>0$ was arbitrary, the claim follows.
	
	\textbf{Step 6} Finally, we show that $k$ is a kernel, i.e.,  $k$ is a symmetric and positive definite function on~$\Pb(X).$
	
	\emph{Symmetry} Let $\mu,\mu^\prime\in\Pb(X)$ and $\vec{x}_M,\vec{x}_M^\prime\in X^M$ such that
	\\
	$\dkm(\muhat[\vec{x}_M],\mu),\dkm(\muhat[\vec{x}_M^\prime],\mu^\prime)\rightarrow 0$.
	For convenience, define $\muhat_{\ell}=\muhat[\vec{x}_{M_{\ell}}]$ and $\muhat_{\ell}^\prime=\muhat[\vec{x}^\prime_{M_{\ell}}]$
	We then have
	\begin{align*}
		|k(\mu,\mu^\prime)-k(\mu^\prime,\mu)| 
		& \leq |k(\mu,\mu^\prime)-k(\muhat_{\ell},\muhat_{\ell}^\prime)| \\
		& \quad + |k(\muhat_{\ell},\muhat_{\ell}^\prime) - k^{[M_{\ell}]}(\vec{x}_{M_{\ell}}, \vec{x}^\prime_{M_{\ell}})| \\
		& \quad + |k^{[M_{\ell}]}(\vec{x}_{M_{\ell}}^\prime, \vec{x}_{M_{\ell}}) - k(\muhat_k^\prime,\muhat_k)| \\
		& \quad  + |k(\muhat_{\ell}^\prime,\muhat_{\ell})-k(\mu^\prime,\mu)| \\
		& \rightarrow 0,
	\end{align*}
	where we used the symmetry of $k^{[M_{\ell}]}$ in the inequality and then the continuity of $k$ (w.r.t. $\dkm^2$) as well as \eqref{eq.conv_kernels}.
	
	\emph{Positive definiteness} Let $N\in\N_+$, $\mathbf{\alpha}\in\R^N$ and $\mu_1,\ldots,\mu_N\in\Pb(X)$ as well as $\vec{x}_n^{[M]}\in X^M$ such that for all $n=1,\ldots,N$, $\dkm(\muhat[\vec{x}_n^{[M]}],\mu_n)\rightarrow 0$.
	For convenience, define $\muhat^{[M]}_n=\muhat[\vec{x}_n^{[M]}]$.
	Let $\epsilon>0$ be arbitrary. For all $i,j=1,\ldots,N$ and $M$ we have
	\begin{align*}
		k(\mu_i,\mu_j) & \geq k(\muhat_i^{[M]},\muhat_j^{[M]}) - |k(\mu_i,\mu_j) - k(\muhat_i^{[M]},\muhat_j^{[M]})| \\
		& \geq k^{[M]}(\vec{x}_i^{[M]},\vec{x}_j^{[M]}) - |k(\muhat_i^{[M]},\muhat_j^{[M]}) - k^{[M]}(\vec{x}_i^{[M]},\vec{x}_j^{[M]})| 
		\\
		& \quad - |k(\mu_i,\mu_j) - k(\muhat_i^{[M]},\muhat_j^{[M]})|
	\end{align*}
	Choosing ${\ell}$ large enough and setting $M=M_{\ell}$ ensures
	\begin{equation*}
		k(\mu_i,\mu_j) \geq k^{[M_k]}(\vec{x}_i^{[M_k]},\vec{x}_j^{[M_{\ell}]}) - 2\epsilon
	\end{equation*}
	due to the continuity of the $k$ and \eqref{eq.conv_kernels}.
	Repeating this for all pairs $(i,j)$ and taking the maximum over all resulting $k$ then leads to
	\begin{align*}
		\sum_{i,j=1}^N \alpha_i \alpha_j k(\mu_i,\mu_j)
		& \geq \sum_{i,j=1}^N \alpha_i \alpha_j  k^{[M_{\ell}]}(\vec{x}_i^{[M_{\ell}]},\vec{x}_j^{[M_{\ell}]}) - 2N^2\epsilon
		\geq -2N^2\epsilon,
	\end{align*}
	where we used that $k^{[M_{\ell}]}$ is a kernel. Since $\epsilon>0$ was arbitrary, we find that
	\begin{equation*}
		\sum_{i,j=1}^N \alpha_i \alpha_j k(\mu_i,\mu_j) \geq 0.
	\end{equation*}
\end{proof}
%
%
\begin{remark}
	The function $\tw_k$ from the proof of Theorem \ref{thm.mfl_kernel} is also a modulus of continuity for $k$, i.e., for all $\mu_i\in\Pb(X)$, $i=1,\ldots,4$,
	\begin{equation*}
		|k(\mu_1,\mu_2)-k(\mu_3,\mu_4)|\leq \tw_k(\dkm^2[(\mu_1,\mu_2), (\mu_3,\mu_4)]).
	\end{equation*}
\end{remark}
\begin{proof}
	Let $\mu_i\in\mathcal{P}(X)$, $i=1,\ldots,4$, and $\epsilon>0$ be arbitrary. Choose $n\in\N_+$ such that
	\begin{equation*}
		\|k^{[M_n]}_\McS - k\|_\infty = \sup_{\mu,\mu^\prime\in\Pb(X)}|k^{[M_n]}_\McS(\mu,\mu^\prime)-k(\mu,\mu^\prime)| \leq \frac\epsilon2
	\end{equation*}
	(exists due to the Arzela-Ascoli Theorem). We then have
	\begin{align*}
		|k(\mu_1,\mu_2)-k(\mu_3,\mu_4)| 
		& \leq |k(\mu_1,\mu_2) - k^{[M_n]}_\McS(\mu_1,\mu_2)| \\
		& \quad  + |k^{[M_n]}_\McS(\mu_1,\mu_2) - k^{[M_n]}_\McS(\mu_3,\mu_4)|
		\\
		& \quad + |k^{[M_n]}_\McS(\mu_3,\mu_4) - k(\mu_3,\mu_4)| \\
		& \leq \frac\epsilon2 + \tw_k(\dkm^2[(\mu_1,\mu_2), (\mu_3,\mu_4)]) + \frac\epsilon2
	\end{align*}
	Since $\epsilon>0$ was arbitrary, we find that
	\begin{equation*}
		|k(\mu_1,\mu_2)-k(\mu_3,\mu_4)|\leq \tw_k(\dkm^2[(\mu_1,\mu_2), (\mu_3,\mu_4)]).
	\end{equation*}
	This finishes the proof. 
\end{proof}

\begin{remark}
	It is also possible to generalize Assumption \ref{assump.mfl_kernel_cond} and Theorem \ref{thm.mfl_kernel} to kernel sequences of the form $k^{[M]}: (Y \times X^M) \times (Y \times X^M) \rightarrow \R$ for some compact metric space $Y$, leading to a mean field kernel $k: (Y\times\Pb(X))\times (Y\times\Pb(X))\rightarrow \R$ using techniques presented for example in \cite{blanchet2014nash}.
\end{remark}

\section{The reproducing {kernel} Hilbert space of the mean field limit kernel}\label{section.mfl_rkhs}

The mean field limit $k$ established above is a kernel and therefore it is associated with a unique RKHS.  The goal of this section is the investigation of elements (functions) in this  RKHS and their relation to the elements of  RKHS  induced by $k^{[M]}$.  In particular, we establish the bottom part of Figure \ref{fig.mfl_kernel_diagram1}.  For brevity, define $H_M = H_{k^{[M]}}$ and $\|\cdot\|_M=\|\cdot\|_{k^{[M]}}$.
%
%
We start by noting the following interesting fact about feature space-feature map pairs for the kernel $k^{[M]}$. For the definition of feature maps, we refer to Section \ref{section.prelim}. 

\begin{proposition}
	For $M\in\N_+$, let $(\mathcal{H}_M,\Phi_M)$ be \emph{any} feature space-feature map pair for $k^{[M]}$.
	\begin{enumerate}
		\item For all $M\in\N_+$, $\Phi_M$ is invariant under permutations, i.e., for all $\vec{x}\in X^M$ and $\sigma \in \Permutations_M$ we have $\Phi_M(\sigma \vec{x})=\Phi_M(\vec{x})$.
		\item For all $M\in\N_+$ and $\vec{x}\in X^M$ we have $\|\Phi_M(\vec{x})\|_{\mathcal{H}_M}\leq \sqrt{C_k}$.
		\item $\sqrt{2\omega_k}$ is a modulus of continuity for $\Phi_M$ for all $M\in\N_+$, i.e., for all $\vec{x}_1,\vec{x}_2\in X^M$
		we have
		\begin{equation*}
			\|\Phi_M(\vec{x}_1)-\Phi_M(\vec{x}_2)\|_{\mathcal{H}_M}
			\leq \sqrt{2\omega_k\left(\dkm\left[ \muhat[\vec{x}_1],\muhat[\vec{x}_2]\right]\right)}
		\end{equation*}
	\end{enumerate}
\end{proposition}
\begin{proof}
	\begin{enumerate}
		\item Let $M\in\N_+$, $\vec{x}\in X^M$ and $\sigma\in\Permutations_M$ be arbitrary. From
		\begin{align*}
			\|\Phi_M(\sigma\vec{x})-\Phi_M(\vec{x})\|_{\mathcal{H}_M}^2
			& = \langle \Phi_M(\sigma\vec{x}, \Phi_M(\sigma\vec{x})\rangle_{\mathcal{H}_M} 
			- 2\langle \Phi_M(\sigma\vec{x}), \Phi_M(\vec{x})\rangle_{\mathcal{H}_M}
			\\ 
			& \quad + \langle \Phi_M(\vec{x}),\Phi_M(\vec{x})\rangle_{\mathcal{H}_M} \\
			&  = k^{[M]}(\sigma\vec{x},\sigma\vec{x}) - 2k^{[M]}(\sigma\vec{x},\vec{x}) + k^{[M]}(\vec{x},\vec{x}) \\
			& = k^{[M]}(\vec{x},\vec{x}) - 2k^{[M]}(\vec{x},\vec{x}) + k^{[M]}(\vec{x},\vec{x}) \\
			& = 0
		\end{align*}
		(where we used the symmetry and permutation invariance of $k^{[M]}$) we find that $\Phi_M(\sigma\vec{x})=\Phi_M(\vec{x})$, hence the permutation invariance of all $\Phi_M$.
		\item Let $M\in\N_+$ and $\vec{x}\in X^M$ be arbitrary, then
		\begin{equation*}
			\|\Phi_M(\vec{x})\|_{\mathcal{H}_M} 
			= \sqrt{\langle \Phi_M(\vec{x}),\Phi_M(\vec{x})\rangle_{\mathcal{H}_M} } 
			= \sqrt{k^{[M]}(\vec{x},\vec{x})} 
			\leq \sqrt{C_k}.
		\end{equation*}
		\item Let $M\in\N_+$ and $\vec{x}_1,\vec{x}_2\in X^M$ be arbitrary, then
		\begin{align*}
			\|\Phi_M(\vec{x}_1)-\Phi_M(\vec{x}_2)\|_{\mathcal{H}_M}^2
			& = \langle \Phi_M(\vec{x}_1), \Phi_M(\vec{x}_1) \rangle_{\mathcal{H}_M}
			- \langle \Phi_M(\vec{x}_2), \Phi_M(\vec{x}_1) \rangle_{\mathcal{H}_M} \\
			& \quad - \langle \Phi_M(\vec{x}_1), \Phi_M(\vec{x}_2) \rangle_{\mathcal{H}_M}
			+ \langle \Phi_M(\vec{x}_2), \Phi_M(\vec{x}_2) \rangle_{\mathcal{H}_M} \\
			& = k^{[M]}(\vec{x}_1,\vec{x}_1) -  k^{[M]}(\vec{x}_2,\vec{x}_1)
			- k^{[M]}(\vec{x}_1,\vec{x}_2) +  k^{[M]}(\vec{x}_2,\vec{x}_2) \\
			& \leq |  k^{[M]}(\vec{x}_1,\vec{x}_1) -  k^{[M]}(\vec{x}_2,\vec{x}_1)|
			+ |k^{[M]}(\vec{x}_1,\vec{x}_2) -  k^{[M]}(\vec{x}_2,\vec{x}_2)| \\
			& \leq 2\omega_k(\dkm(\muhat[\vec{x}_1], \muhat[\vec{x}_2])),
		\end{align*}
		hence
		\begin{equation*}
			\|\Phi_M(\vec{x}_1)-\Phi_M(\vec{x}_2)\|_{\mathcal{H}_M}
			\leq \sqrt{2\omega_k(\dkm(\muhat[\vec{x}_1], \muhat[\vec{x}_2]))}.
		\end{equation*}
	\end{enumerate}
\end{proof}
%
%
Next, we investigate  properties of functions $f\in H_M$ where $H_M$ is the RKHS corresponding to $k^{[M]}$.
\begin{proposition} \label{prop.mfl_rkhs_cond}
	Let $M\in\N_+$ and $f\in H_M$ be arbitrary.
	\begin{enumerate}
		\item For all $\vec{x} \in X^M$ and $\sigma\in\Permutations_M$ we have
		\begin{equation*}
			f(\sigma \vec{x})=f(\vec{x}).
		\end{equation*}
		\item For all $\vec{x} \in X^M$ we get
		\begin{equation*}
			|f(\vec{x})| \leq \|f\|_{H_M} \sqrt{C_k}.
		\end{equation*}
		\item Let $\vec{x}_1,\vec{x}_2\in X^M$ be arbitrary, then
		\begin{equation*}
			|f(\vec{x}_1)-f(\vec{x}_2)|\leq \sqrt{2\omega_k(\dkm(\muhat[\vec{x}_1],\muhat[\vec{x}_2]))}.
		\end{equation*}
	\end{enumerate}
\end{proposition}
The arguments used in the proof are standard, but for completeness we provide all details.
\begin{proof}
	Using the reproducing property and symmetry of $k^{[M]}$, we find for $\vec x \in X^M$ and $\sigma\in\Permutations_M$
	\begin{equation*}
		f(\sigma \vec{x}) =\langle f, k^{[M]}(\sigma \vec{x},\cdot)\rangle_{H_M} 
		= \langle f, k^{[M]}(\vec{x},\cdot)\rangle_{H_M}
		= f(\vec{x}),
	\end{equation*}
	establishing the first claim.
	Next, using again the reproducing property of $k^{[M]}$, Cauchy-Schwarz and the boundedness of $k^{[M]}$ we get
	\item \begin{align*}
		|f(\vec{x})|  = |\langle f, k^{[M]}(\vec{x},\cdot) \rangle_{H_M}| 
		\leq \|f\|_{H_M} \|k^{[M]}(\vec{x},\cdot)\|_{H_M} 
		=  \|f\|_{H_M} \sqrt{k^{[M]}(\vec{x},\vec{x})} 
		\leq  \|f\|_{H_M} \sqrt{C_k},
	\end{align*}
	showing the second statement.
	Similarly, for $\vec{x}_1,\vec{x}_2\in X^M$ we get
	\item \begin{align*}
		|f(\vec{x}_1)-f(\vec{x}_2)| & = |\langle f, k^{[M]}(\vec{x}_1, \cdot)-k^{[M]}(\vec{x}_2, \cdot)\rangle_{H_M}| 
		\leq \|f\|_{H_M} \|k^{[M]}(\vec{x}_1,\cdot)-k^{[M]}(\vec{x}_2, \cdot)\|_{H_M} \\
		& = \|f\|_{H_M} \sqrt{k^{[M]}(\vec{x}_1,\vec{x}_1)-k^{[M]}(\vec{x}_1,\vec{x}_2) + k^{[M]}(\vec{x}_2,\vec{x}_2) - k^{[M]}(\vec{x}_2,\vec{x}_1)} \\
		& \leq \|f\|_{H_M} \sqrt{|k^{[M]}(\vec{x}_1,\vec{x}_1)-k^{[M]}(\vec{x}_1,\vec{x}_2)| 
			+ |k^{[M]}(\vec{x}_2,\vec{x}_2) - k^{[M]}(\vec{x}_2,\vec{x}_1)|} \\
		& \leq \|f\|_{H_M} \sqrt{2\omega_k(\dkm(\muhat[\vec{x}_1],\muhat[\vec{x}_2]))}.
	\end{align*}
\end{proof}
If a sequence $(f_M)_M$, $f_M\in H_M$, is uniformly bounded in the norm of $H_M$, then the second statement in Proposition \ref{prop.mfl_rkhs_cond} ensures that this sequence is point-wise bounded and the third statement implies that the sequence is equi--continuous.
This shows that Assumption \ref{assump.cardaliaguet2010} is fulfilled and Theorem \ref{thm.mfl_funcs} applies. Hence, the following  corollary holds true.
\begin{corollary}
	Let $f_M\in H_M$ with $\|f_M\|_M \leq B$ for some $B\in\Rnn$. Then there exists a subsequence $(f_{M_\ell})_\ell$ and some $f: \Pb(X)\rightarrow\Rnn$, such that
	\begin{equation*}
		\lim_{\ell\rightarrow\infty} \sup_{\vec x\in X^{M_\ell}} |f_{M_\ell}(\vec x) - f(\muhat[\vec x])| = 0.
	\end{equation*}
\end{corollary}

The link in Figure \ref{fig.mfl_kernel_diagram1} between RKHS functions and mean field limits of RKHS functions on $H_M$ for $M\to \infty$ will now be established. 

\begin{theorem} \label{thm.mfl_rkhs_funcs}
	For each $f\in H_k$ there exists a subsequence $({M_\ell^{(2)}})_{\ell}$ of $(M_{\ell})_{\ell}$ and functions $f_{M_\ell^{(2)}}\in H_{M_\ell^{(2)}}$ such that
	\begin{equation*}
		\lim_{{\ell}\rightarrow\infty} \sup_{\vec{x}\in X^{M_\ell^{(2)}}} |f_{M^{(2)}_{\ell}}(\vec{x})-f(\muhat[\vec{x}])| = 0.
	\end{equation*}
\end{theorem}
\begin{proof}
	Let $(\epsilon_{\ell})_{\ell}$ such that $\epsilon_{\ell}>0$ and $\epsilon_{\ell} \searrow 0$. Let $f\in H_k$ be arbitrary.
	
	\textbf{Step 1} For each ${\ell}\in\N_+$ choose
	\begin{equation*}
		{f_\ell^\text{pre}} = \sum_{{n}=1}^{{ N_\ell}}  {\alpha_n^{(\ell)}} k(\cdot, \mu^{({\ell})}_{n}) \in H_k^\text{pre}
	\end{equation*}
	with
	\begin{equation*}
		\|f-{f_\ell^\text{pre}}\|_k\leq \frac{\epsilon_{\ell}}{3\sqrt{C_k}}.
	\end{equation*}
	Such functions exist since $H_k^\text{pre}$ is dense in $H_k$ (recall that $H_k$ is the RKHS of kernel $k$, cf. Theorem \ref{thm.mfl_kernel}).
	Next, again for each ${\ell}\in\N_+$, choose ${L_0^{(\ell)}}\in\N_+$ such that for all ${j \geq L_0^{(\ell)}}$ we have
	\begin{equation*}
		\sup_{\vec{x},{ \vec{x}'}\in X^{M_{j}}} | k^{[M_{j}]}(\vec{x},\vec{x}^\prime)-k(\muhat[\vec{x}],\muhat[\vec{x}^\prime])| \leq \frac{\epsilon_{\ell}}{3(|\alpha_1^{({\ell})}| + \ldots + |\alpha_{{N_\ell}}^{({\ell})}| + 1)},
	\end{equation*}
	and such ${L_0^{(\ell)}}$ exist due to \eqref{eq.conv_kernels}.
	Additionally, for each ${n}=1,\ldots,{N_\ell}$ choose a sequence { $\left(\vec{x}^{(\ell,n)}_M\right)_M$, $\vec{x}^{(\ell,n)}_M\in X^M$,} with $\dkm(\muhat[{\vec{x}^{(\ell,n)}_M}],{\mu^{(\ell)}_n})\rightarrow 0$ for ${M}\rightarrow\infty$.
	Furthermore, for each ${\ell}\in\N_+$ and ${n}=1,\ldots,{N_\ell}$ choose some ${L_n^{(\ell)}}\in\N_+$ such that for all $M \geq M_{L_n^{(\ell)}}$ we have
	\begin{equation*}
		\dkm(\muhat[{\vec{x}^{(\ell,n)}_M}],{\mu^{(\ell)}_n}) 
		\leq \tw_k^{-1}\left( \frac{\epsilon_{\ell}}{3(|{\alpha_1^{(\ell)}}| + \ldots + |{\alpha_{N_\ell}^{(\ell)}}|+1)} \right){.}
	\end{equation*}
	Such ${ L_n^{(\ell)}}$ exist since the right hand side is fixed for given $\ell$ and $n$, and due to the convergence of $\muhat[{\vec{x}^{(\ell,n)}_M}]$ {to $\mu^{(\ell)}_n$}.
	Finally, define { $L_1=\max\{ L_0^{(1)}, L_1^{(1)},\ldots,L_{N_1}^{(1)} \}$ and for $\ell\geq 2$
		\begin{align*}
			L_\ell & = \max\left\{ L_1, \ldots, L_{\ell-1}, \max\{ L_0^{(\ell)}, L_1^{(\ell)},\ldots,L_{N_\ell}^{(\ell)} \} \right\}\\
			M_\ell^{(2)} & = M_{L_\ell}
	\end{align*} }
	and
	\begin{align*}
		{\hat{f}^\text{pre}_{M_\ell^{(2)}}} & = \sum_{{n}=1}^{{N_\ell}}  {\alpha_n^{(\ell)}} k(\cdot, \muhat[{\vec{x}_{M_\ell^{(2)}}^{(\ell,n)}}]) \\
		f_{M_\ell^{(2)}} & = \sum_{{n}=1}^{{N_\ell}}  {\alpha_n^{(\ell)}} k^{[{M_\ell^{(2)}}]}(\cdot, {\vec{x}_{M_\ell^{(2)}}^{(\ell,n)}}) {.}
	\end{align*}
	
	\textbf{Step 2} Let ${\ell}\in\N_+$ and $\vec{x}\in X^{M_\ell^{(2)}}$ be arbitrary.
	We have
	\begin{align*}
		|f_{M_\ell^{(2)}}(\vec{x}) - f(\muhat[\vec{x}])|
		& \leq | f(\muhat[\vec{x}]) - {f^\text{pre}_{M_\ell^{(2)}}}(\muhat[\vec{x}])|
		+ |{f^\text{pre}_{M_\ell^{(2)}}}(\muhat[\vec{x}]) - {\hat{f}^\text{pre}_{M_\ell^{(2)}}}(\muhat[\vec{x}]) |
		\\
		& \quad + |{\hat{f}^\text{pre}_{M_\ell^{(2)}}}(\muhat[\vec{x}]) - f_{M_\ell^{(2)}}(\vec{x}) | \\
		& = I + II + III
	\end{align*}
	and continue with
	\begin{align*}
		I & = |\langle f - {f_\ell^\text{pre}}, k(\cdot, \muhat[\vec{x}]) \rangle_k| 
		\leq \| f - {f_\ell^\text{pre}} \|_k \| k(\cdot, \muhat[\vec{x}]) \|_k 
		\leq \frac{\epsilon_{\ell}}{3\sqrt{C_k}} \cdot \sqrt{k(\muhat[\vec{x}],\muhat[\vec{x}])} 
		\leq \frac{\epsilon_{\ell}}{3},
	\end{align*}
	using the reproducing property of $k$, Cauchy-Schwarz, the choice of ${f_\ell^\text{pre}}$ (and again the reproducing property of $k$ together with the definition of $\|\cdot\|_k$) and finally the boundedness of $k$.
	Next,
	\begin{align*}
		II & = \left|
		\sum_{{n}=1}^{{N_\ell}} {\alpha_n^{(\ell)}} k(\muhat[\vec{x}], {\mu^{(\ell)}_n})
		-
		\sum_{{n}=1}^{{N_\ell}} {\alpha_n^{(\ell)}} k(\muhat[\vec{x}], \muhat[{\vec{x}_{M_\ell^{(2)}}^{(\ell,n)}}])
		\right| \\
		& \leq \sum_{{n}=1}^{{N_\ell}} |  {\alpha_n^{(\ell)}}| | k(\muhat[\vec{x}], {\mu^{(\ell)}_n}) - k(\muhat[\vec{x}], \muhat[{\vec{x}_{M_\ell^{(2)}}^{(\ell,n)}}])| 
		\\
		& \leq  \sum_{{n}=1}^{{N_\ell}} |  {\alpha_n^{(\ell)}}|  \tw_k\left( 
		\dkm^2[(\muhat[\vec{x}], \muhat[\vec{x}]), ({\mu^{(\ell)}_n}, \muhat[{\vec{x}_{M_\ell^{(2)}}^{(\ell,n)}}])]
		\right) \\
		& = \sum_{{n}=1}^{{N_\ell}} |  {\alpha_n^{(\ell)}}|  \tw_k\left( 
		\dkm({\mu^{(\ell)}_n},  \muhat[{\vec{x}_{M_\ell^{(2)}}^{(\ell,n)}}])
		\right)  
		\leq \sum_{{n}=1}^{{N_\ell}} |  {\alpha_n^{(\ell)}}| \frac{\epsilon_{\ell}}{3(|{\alpha_1^{(\ell)}}| + \ldots + |{\alpha_{N_\ell}^{(\ell)}}|+1)} 
		\leq \frac{\epsilon_{\ell}}{3},
	\end{align*}
	where we used the definition of $f^\text{pre}_{M_\ell^{(2)}}$ and $\hat{f}^\text{pre}_{M_\ell^{(2)}}$, respectively,
	the triangle inequality,
	and the choice of $\vec{x}^{(\ell,n)}_M$ and $M_\ell^{(2)}$.
	Finally,
	\begin{align*}
		III & = \left| 
		\sum_{{n}=1}^{{N_\ell}} {\alpha_n^{(\ell)}} k(\muhat[\vec{x}], \muhat[{\vec{x}_{M_\ell^{(2)}}^{(\ell,n)}}])
		-
		\sum_{{n}=1}^{{N_\ell}}  {\alpha_n^{(\ell)}} k^{[{M_\ell^{(2)}}]}(\vec{x}, {\vec{x}_{M_\ell^{(2)}}^{(\ell,n)}})
		\right|  \\
		&\leq  \sum_{{n}=1}^{{N_\ell}} | {\alpha_n^{(\ell)}}| 
		|k(\muhat[\vec{x}], \muhat[{\vec{x}_{M_\ell^{(2)}}^{(\ell,n)}}])
		- 
		k^{[{M_\ell^{(2)}}]}(\vec{x},{\vec{x}_{M_\ell^{(2)}}^{(\ell,n)}})| \\
		& \leq  \sum_{{n}=1}^{{N_\ell}} | {\alpha_n^{(\ell)}}|  \left(
		\sup_{\vec{x}_1,\vec{x}_2\in X^{{M_\ell^{(2)}}}} |k^{[{M_\ell^{(2)}}]}(\vec{x}_1,\vec{x}_2) - k(\muhat[\vec{x}_1],\muhat[\vec{x}_2])|
		\right) \\
		& \leq  \sum_{{n}=1}^{{N_\ell}} | {\alpha_n^{(\ell)}}| \frac{\epsilon_{\ell}}{3(|{\alpha_1^{(\ell)}}| + \ldots + |{\alpha_{N_\ell}^{(\ell)}}| + 1)}  \leq \frac{\epsilon_{\ell}}{3}.
	\end{align*}
	Altogether,
	\begin{equation*}
		|f_{M_\ell^{(2)}}(\vec{x}) - f(\muhat[\vec{x}])| \leq \epsilon_{\ell}
	\end{equation*}
	{ for all $\ell\in\N_+$}, and since $\vec{x}\in X^{M_\ell^{(2)}}$ was arbitrary, we get
	\begin{equation*}
		\sup_{\vec{x}\in X^{M_\ell^{(2)}}} | f_{M_\ell^{(2)}}(\vec{x}) - f(\muhat[\vec{x}])| \leq \epsilon_{\ell}
	\end{equation*}
	which in turn implies together with $\epsilon_{\ell}\searrow 0$ that
	\begin{equation*}
		\lim_{{\ell}\rightarrow\infty} \sup_{\vec{x}\in X^{M_\ell^{(2)}}} |f_{M^{(2)}_\ell}(\vec{x})-f(\muhat[\vec{x}])| = 0.
	\end{equation*}
\end{proof}
Summarizing,  a generic RKHS function $f\in H_k$ is obtained by the following procedure: Consider the mean field limit of the $k^{[M]}$ to obtain  $k$ and then form its RKHS $H_k$. Equivalently,  we may form the RKHS $H_M$ for each $k^{[M]}$ and then go to the mean field limit of a suitable (sub)sequence of RKHS functions $f_M\in H_M$.

\section{Examples} \label{section.examples}
We now introduce two large classes of concrete kernel sequences that are suitable for the mean field limit as outlined in the previous two sections.
\subsection{Pullback kernels}
Our first example are sequences of kernels that arise as the pull-backs \cite[Section~5.4]{paulsen2016introduction} of a sufficiently regular kernel along mean field compatible functions. 
\begin{proposition} \label{prop.pullback}
	Let $Y$ be a Banach space, $k_0: Y \times Y \rightarrow \R$ be a kernel on $Y$ and $\phi^{[M]}: X^M \rightarrow Y$ a sequence of functions. Furthermore, assume that
	\begin{enumerate}
		\item \emph{(Boundedness of $k_0$)} There exists a $C_{k_0}\in\Rnn$ with $|k_0(y,y^\prime)|\leq C_{k_0}$ for all $y,y'\in Y$.
		\item \emph{(Continuity of $k_0$)} The kernel $k_0$ has a modulus of continuity $\omega_{k_0}$, i.e., 
		\begin{equation*}
			|k_0(y_1,y_1^\prime)-k_0(y_2,y_2^\prime)|\leq \omega_{k_0}(\|y_1-y_2\|_Y + \|y_1'-y_2'\|_Y)
		\end{equation*}
		for all $y_1,y_1',y_2,y_2'\in Y$.
		\item \emph{(Symmetry of $\phi^{[M]}$)} For all $M\in\N$, the function $\phi^{[M]}$ is permutation invariant, i.e., for all $\vec{x}\in X^M$ and $\sigma\in \Permutations_M$ we have $\phi^{[M]}(\sigma \vec x)=\phi^{[M]}(\vec x)$.
		\item \emph{(Uniform continuity of $\phi^{[M]}$)} There exists a modulus of continuity $\omega_\phi:\Rnn\rightarrow\Rnn$ such that for all $M\in\N_+$, $\vec{x},\vec{x}^\prime\in X^M$
		\begin{equation*}
			\|\phi^{[M]}(\vec{x})-\phi^{[M]}(\vec{x}')\|_Y
			\leq \omega_\phi\left( \dkm(\muhat[\vec{x}],\muhat[\vec{x}^\prime])\right).
		\end{equation*}
	\end{enumerate}
	Then $k^{[M]}: X^M \times X^M \rightarrow \R$ with $k^{[M]}(\vec x, \vec x')=k_0(\phi^{[M]}(\vec x), \phi^{[M]}(\vec x'))$ is a sequence of kernels on $X^M$ fulfilling Assumption \ref{assump.mfl_kernel_cond}.
\end{proposition}
\begin{proof}
	Since $k^{[M]}$ is the pull-back of $k_0$ along $\phi^{[M]}$, it is a kernel on $X^M$.
	Symmetry is clear,
	\begin{equation*}
		k^{[M]}(\sigma \vec x, \vec x') = k_0(\phi^{[M]}(\sigma \vec x), \phi^{[M]}(\vec x')) 
		= k_0(\phi^{[M]}(\vec x), \phi^{[M]}(\vec x')) = k^{[M]}(\vec x, \vec x').
	\end{equation*}
	Uniform boundedness follows from boundedness of $k_0$, hence $C_k=C_{k_0}$.
	For the uniform continuity, let $M\in\N_+$, $\vec{x}_1,\vec{x}_1^\prime,\vec{x}_2,\vec{x}_2'\in X^M$, then
	\begin{align*}
		|k^{[M]}(\vec{x}_1,\vec{x}_1^\prime)-k^{[M]}(\vec{x}_2,\vec{x}_2^\prime)|
		& = |k_0(\phi^{[M]}(\vec{x}_1),\phi^{[M]}(\vec{x}_1^\prime))-k_0(\phi^{[M]}(\vec{x}_2),\phi^{[M]}(\vec{x}_2^\prime))| \\
		& \leq \omega_{k_0}\left( 
		\| \phi^{[M]}(\vec x_1) - \phi^{[M]}(\vec x_2)\|_Y
		+ \| \phi^{[M]}(\vec x_1') - \phi^{[M]}(\vec x_2')\|_Y
		\right) \\
		& \leq \omega_{k_0}\left( 
		\omega_\phi(\dkm(\muhat[\vec x_1], \muhat[\vec x_2]) )
		+\omega_\phi(\dkm(\muhat[\vec x_1], \muhat[\vec x_2']) )
		\right) \\
		& \leq \omega_k \left( \dkm^2\left[ (\muhat[\vec{x}_1],\muhat[\vec{x}_1^\prime),  (\muhat[\vec{x}_2],\muhat[\vec{x}_2^\prime)\right] \right)
	\end{align*}
	for an appropriate modulus of continuity $\omega_k$.
\end{proof}

\subsection{Double-sum kernels}
The next class of examples has been introduced by \cite{kim2021bayesian} and extended by \cite{buathong2020kernels}, though similar constructions have been used earlier \cite{gartner2002multi}.
However, the connection to mean field limits and kernel mean embeddings has not yet been investigated.
%
%
\begin{proposition} \label{prop.ds_mfl}
	Let $k_0: X \times X \rightarrow \R$ be a kernel bounded by $|k_0(x,x')|\leq C_{k_0}$ for some $C_{k_0}\in\Rnn$.
	Define for $M\in\Np$ the map $k^{[M]}: X^M \times X^M \rightarrow \R$ by
	\begin{equation} \label{eq.double_sum_kernel}
		k^{[M]}(\vec{x},\vec{x}') = \frac{1}{M^2} \sum_{m,m'=1}^M k_0(x_m, x_{m'}').
	\end{equation}
	Then $k^{[M]}$ are kernels that are permutation invariant in their first argument, and that are uniformly bounded.
\end{proposition}
\begin{proof}
	Let $M\in\Np$ be arbitrary. First, we establish that $k^{[M]}$ is indeed a kernel by showing that it is a symmetric, positive definite function. Note that this fact has been established earlier, cf. e.g. \cite{buathong2020kernels}, but for convenience we provide a full proof.
	For all $\vec{x},\vec{x}'\in X^M$ we have (using the symmetry of $k$)
	\begin{equation*}
		k^{[M]}(\vec{x},\vec{x}') = \frac{1}{M^2} \sum_{m,m'=1}^M k_0(x_m, x_{m'}')
		= \frac{1}{M^2} \sum_{m,m'=1}^M k_0(x_{m'}', x_m) = k^{[M]}(\vec{x}', \vec{x}) ,
	\end{equation*}
	i.e., $k^{[M]}$ is symmetric. Let $N\in\Np$ and $\vec{x}^1,\ldots,\vec{x}^N\in X^M$, $\alpha\in \R^N$ be arbitrary, then
	\begin{align*}
		\sum_{i,j=1}^N \alpha_i \alpha_j k^{[M]}(\vec{x}^i, \vec{x}^j) 
		& = \sum_{i,j=1}^N \alpha_i \alpha_j \frac{1}{M^2} \sum_{m,m'=1}^M k_0(x^i_m, x^j_{m'}) \\
		& = \sum_{i,j=1}^N \sum_{m,m'=1}^M \alpha_i \alpha_j \frac{1}{M^2} k_0(x^i_m, x^j_{m'}) \\
		& = \sum_{(i,m),(j,m')\in \mathcal{I}} \frac{\alpha_i}{M} \frac{\alpha_j}{M} k_0(x^i_m, x^j_{m'}) \geq 0,
	\end{align*}
	where we defined $\mathcal{I}=\{1,\ldots,N\}\times\{1,\ldots,M\}$ and used that $k_0$ is positive definite.
	
	For the uniform boundedness, let $\vec{x},\vec{x}'\in X^M$, then
	\begin{equation*}
		|k^{[M]}(\vec{x}, \vec{x}')| \leq \frac{1}{M^2} \sum_{m,m'=1}^M |k_0(x_m, x_{m'}')| \leq \frac{1}{M^2} M^2 C_{k_0} = C_{k_0}.
	\end{equation*}
\end{proof}
In addition to permutation-invariance and boundedness, we also have a form of uniform continuity of double sum kernels.
\begin{proposition} \label{prop.uniform_continuity_ds}
	Let $k_0: X \times X \rightarrow \R$ be a kernel bounded by $|k_0(x,x')|\leq C_{k_0}$ for some $C_{k_0}\in\Rnn$, and assume that $(X,d_{k_0})$ is a separable metric space, where
	\begin{equation*}
		d_{k_0}: X \times X \rightarrow \Rnn, \: d_{k_0}(x,x')=\|\Phi_{k_0}(x)-\Phi_{k_0}(x')\|_{k_0}
	\end{equation*}
	is the usual kernel metric, cf. Section \ref{section.rkhs}. Then the double sum kernels $k^{[M]}$ defined in \eqref{eq.double_sum_kernel} are uniformly continuous with respect to the Kantorowich-Rubinstein distance induced by $d_{k_0}$.
\end{proposition}
\begin{proof}
	Observe that for $\vec{x},\vec{x}'\in X^M$  we have
	\begin{align*}
		k^{[M]}(\vec{x},\vec{x}') & = \frac{1}{M^2} \sum_{m,m'=1}^M k_0(x_m, x_{m'}')  =  \frac{1}{M^2} \sum_{m,m'=1}^M \langle k_0(\cdot, x_{m'}'), k_0(\cdot, x_m)\rangle_{k_0}  \\
		& = \left\langle \frac{1}{M} \sum_{m'=1}^M k_0(\cdot, x_{m'}'), \frac{1}{M} \sum_{m=1}^M k_0(\cdot,x_m)\right\rangle_{k_0}  = \langle f_{\muhat[\vec{x}']}, f_{\muhat[\vec{x}]} \rangle_{k_0}.
	\end{align*}

	Furthermore, we also have for any $\vec x\in X^M$
	\begin{align*}
		\|f^{k_0}_{\muhat[\vec x]}\|_{k_0}  & = \sqrt{ \langle \int k_0(\cdot,x)\mathrm{d}\muhat[\vec x](x),\int k_0(\cdot,x')\mathrm{d}\muhat[\vec x](x') \rangle_{k_0} } 
		\\
		& = \sqrt{\int \int \langle k_0(\cdot,x), k_0(\cdot, x') \rangle_{k_0} \mathrm{d}\muhat[\vec x](x)\mathrm{d}\muhat[\vec x](x')} \\
		& \leq \sqrt{ \int \int |k_0(x',x)|\mathrm{d}\muhat[\vec x](x)\mathrm{d}\muhat[\vec x](x')}  \leq \sqrt{C_{k_0}}.
	\end{align*}
	Let $\vec{x}_1,\vec{x}_2,\vec{x}_1',\vec{x}_2'\in X^M$, then
	\begin{align*}
		|k^{[M]}(\vec{x}_1,\vec{x}_1') - k^{[M]}(\vec{x}_2, \vec{x}_2')|
		& = |\langle f^{k_0}_{\muhat[\vec{x}_1']}, f^{k_0}_{\muhat[\vec{x}_1]} \rangle_{k_0} - \langle f^{k_0}_{\muhat[\vec{x}_2']}, f^{k_0}_{\muhat[\vec{x}_2]} \rangle_{k_0}| \\
		& = |\langle f^{k_0}_{\muhat[\vec{x}_1']} - f^{k_0}_{\muhat[\vec{x}_2']} , f^{k_0}_{\muhat[\vec{x}_1]}\rangle_{k_0} 
		\\
		& \quad + \langle f^{k_0}_{\muhat[\vec{x}_2']}, f^{k_0}_{\muhat[\vec{x}_1]} -  f^{k_0}_{\muhat[\vec{x}_2]}  \rangle_{k_0}| \\
		& \leq \| f^{k_0}_{\muhat[\vec{x}_1']} -  f^{k_0}_{\muhat[\vec{x}_2']}\|_{k_0} \| f^{k_0}_{\muhat[\vec{x}_1]}\|_{k_0}
		\\
		& \quad + \| f^{k_0}_{\muhat[\vec{x}_2']}\|_{k_0} \|f^{k_0}_{\muhat[\vec{x}_1]} -  f^{k_0}_{\muhat[\vec{x}_2]} \|_{k_0} \\
		& \leq \sqrt{C_k}(\| f^{k_0}_{\muhat[\vec{x}_1']} -  f^{k_0}_{\muhat[\vec{x}_2']}\|_{k_0} + \|f^{k_0}_{\muhat[\vec{x}_1]} -  f^{k_0}_{\muhat[\vec{x}_2]} \|_{k_0})
	\end{align*}
	Next, since $(X,d_{k_0})$ is separable, \cite[Theorem~21]{sriperumbudur2010hilbert} shows that 
	$\|f^{k_0}_{\muhat[\vec{x}_1]} -  f^{k_0}_{\muhat[\vec{x}_2]} \|_{k_0} \leq \widetilde \dkm(\muhat[\vec{x}_1],\muhat[\vec{x}_2])$
	and
	$\|f^{k_0}_{\muhat[\vec{x}'_1]} -  f^{k_0}_{\muhat[\vec{x}'_2]} \|_{k_0} \leq \widetilde\dkm(\muhat[\vec{x}'_1],\muhat[\vec{x}'_2])$,
	where
	\begin{equation*}
		\widetilde\dkm(\mu_1,\mu_2)= \sup\left\{ \int_X \phi(x) \mathrm{d}(\mu_1-\mu_2)(x) \mid \phi: X \rightarrow \R \text{ is 1-Lipschitz w.r.t. $d_{k_0}$} \right\},
	\end{equation*}
	the Kantorowich-Rubinstein distance induced by $d_{k_0}$.
	Altogether, we find that
	\begin{equation*}
		|k^{[M]}(\vec{x}_1,\vec{x}_1') - k^{[M]}(\vec{x}_2, \vec{x}_2')| \leq \sqrt{C_{k_0}}(\widetilde\dkm(\muhat[\vec{x}_1],\muhat[\vec{x}_2]) + \widetilde\dkm(\muhat[\vec{x}'_1],\muhat[\vec{x}'_2])),
	\end{equation*}
	but since $\sqrt{C_{k_0}}$ does not depend on $M$, this establishes uniform continuity of $k^{[M]}$ w.r.t. $\widetilde\dkm$.
\end{proof}
%
%
\begin{remark} \label{remark.ds_metric}
	In Proposition \ref{prop.uniform_continuity_ds}, we have not established uniform continuity of the double sum kernels \eqref{eq.double_sum_kernel} with respect to the Kantorowich-Rubinstein distance induced by the metric $d_X$. In particular, combining Propositions \ref{prop.ds_mfl} and \ref{prop.uniform_continuity_ds} is not enough to ensure that the double sum kernels fulfill Assumption \ref{assump.mfl_kernel_cond}.
	However, if $(X,d_{k_0})$ is a compact, separable metric space, then Proposition \ref{prop.uniform_continuity_ds} implies that Assumption \ref{assump.mfl_kernel_cond}, now with $d_{k_0}$ instead of $d_X$, applies to the kernel sequence \eqref{eq.double_sum_kernel}.
	In this case, Theorem \ref{thm.mfl_kernel} shows the existence of the mean field limit kernel and its associated RKHS, again with $d_{k_0}$ instead of $d_X$.
\end{remark}

Recall from the proof of Proposition \ref{prop.uniform_continuity_ds} that for all $M\in\N_+$ and $\vec{x},\vec{x}'\in X^M$  we have
\begin{align*}
	k^{[M]}(\vec{x},\vec{x}') & = \frac{1}{M^2} \sum_{m,m'=1}^M k_0(x_m, x_{m'}')  =  \frac{1}{M^2} \sum_{m,m'=1}^M \langle k_0(\cdot, x_{m'}'), k_0(\cdot, x_m)\rangle_{k_0}  \\
	& = \left\langle \frac{1}{M} \sum_{m'=1}^M k_0(\cdot, x_{m'}'), \frac{1}{M} \sum_{m=1}^M k_0(\cdot,x_m)\right\rangle_{k_0}  = \langle f^{k_0}_{\muhat[\vec{x}']}, f^{k_0}_{\muhat[\vec{x}]} \rangle_{k_0}.
\end{align*}
This equality implies  that for all $M\in\N_+$ the RKHS $H_0$ is a feature space and $\Phi_M: X^M\rightarrow H_0$,
$\Phi_M(\vec{x})=f^{k_0}_{\muhat[\vec{x}]}$ is a feature map for $k^{[M]}$.
Furthermore, defining $e^{[M]}(x)=(x \cdots x)\in X^M$ for $x\in X$ and $M\in\N_+$, we obtain that for all $M\in\N_+$, $\vec{x}\in X^M$ and $\bar{x}\in X$
\begin{equation*}
	\Phi^{[M]}(\vec{x})(e^{[M]}(\bar{x})) = k^{[M]}(e^{[M]}(\bar{x}),\vec{x}) = f^{k_0}_{\muhat[\vec{x}]}(\bar{x})
\end{equation*}
Note that $e^{[M]}(\bar{x})$ can be interpreted as a representation of $\delta_{\bar{x}}$ in $X^M$ since
$\muhat[e^{[M]}(\bar{x})]=\delta_{\bar{x}}$.
Altogether, we have now two different kernel-based embeddings of empirical probability distributions: We can embed $\muhat[\vec{x}]$ into $H_0$ via the kernel mean embedding  $f^{k_0}_{\muhat[\vec{x}]}$ or we can identify $\muhat[\vec{x}]$ with $\vec{x}$ and embed  into $H_M$ with the canonical feature map $\Phi^{[M]}(\vec{x})=k^{[M]}(\cdot,\vec{x})$. Those two embeddings are connected by evaluations on a Dirac distribution, represented by $\bar{x}\in X$ and $e^{[M]}(\bar{x})\in X^M$, respectively.  This leads to the commutative diagram in Figure \ref{fig.ds_kme_diagram}.
\begin{figure}\center
	\tikzset{every picture/.style={line width=0.75pt}} 
	\begin{tikzpicture}[x=0.75pt,y=0.75pt,yscale=-1,xscale=1]
		
		\draw    (393.27,46.36) -- (393.27,98.29) ;
		\draw [shift={(393.27,100.29)}, rotate = 270] [color={rgb, 255:red, 0; green, 0; blue, 0 }  ][line width=0.75]    (10.93,-3.29) .. controls (6.95,-1.4) and (3.31,-0.3) .. (0,0) .. controls (3.31,0.3) and (6.95,1.4) .. (10.93,3.29)   ;
		\draw    (110.16,41.36) -- (110.16,93.29) ;
		\draw [shift={(110.16,95.29)}, rotate = 270] [color={rgb, 255:red, 0; green, 0; blue, 0 }  ][line width=0.75]    (10.93,-3.29) .. controls (6.95,-1.4) and (3.31,-0.3) .. (0,0) .. controls (3.31,0.3) and (6.95,1.4) .. (10.93,3.29)   ;
		\draw    (178.71,114.14) -- (353.5,114.99) ;
		\draw [shift={(355.5,115)}, rotate = 180.28] [color={rgb, 255:red, 0; green, 0; blue, 0 }  ][line width=0.75]    (10.93,-3.29) .. controls (6.95,-1.4) and (3.31,-0.3) .. (0,0) .. controls (3.31,0.3) and (6.95,1.4) .. (10.93,3.29)   ;
		\draw    (176.71,24.14) -- (351.5,24.99) ;
		\draw [shift={(353.5,25)}, rotate = 180.28] [color={rgb, 255:red, 0; green, 0; blue, 0 }  ][line width=0.75]    (10.93,-3.29) .. controls (6.95,-1.4) and (3.31,-0.3) .. (0,0) .. controls (3.31,0.3) and (6.95,1.4) .. (10.93,3.29)   ;
		
		\draw (91.67,20.43) node    {$\textcolor[rgb]{0.29,0.56,0.89}{X}\textcolor[rgb]{0.29,0.56,0.89}{^{M}}\textcolor[rgb]{0.29,0.56,0.89}{\ni }\vec{x} \ \simeq \hat{\mu }[\vec{x}]\textcolor[rgb]{0.29,0.56,0.89}{\in }\mathcal{\textcolor[rgb]{0.29,0.56,0.89}{P}}\textcolor[rgb]{0.29,0.56,0.89}{(}\textcolor[rgb]{0.29,0.56,0.89}{\vec{x}}\textcolor[rgb]{0.29,0.56,0.89}{)} \ $};
		\draw (402.21,20.64) node    {$f_{\hat{\mu }[\vec{x}]}^{k_{0}}\textcolor[rgb]{0.29,0.56,0.89}{\in H}\textcolor[rgb]{0.29,0.56,0.89}{_{0}}$};
		\draw (112.67,113.71) node    {$\textcolor[rgb]{0.29,0.56,0.89}{H}\textcolor[rgb]{0.29,0.56,0.89}{_{M}}\textcolor[rgb]{0.29,0.56,0.89}{\ni } \ k^{[ M]}( \cdot ,\vec{x})$};
		\draw (260.7,11) node  [font=\scriptsize,color={rgb, 255:red, 208; green, 2; blue, 27 }  ,opacity=1 ]  {KME};
		\draw (85.74,65) node  [font=\scriptsize,color={rgb, 255:red, 208; green, 2; blue, 27 }  ,opacity=1 ]  {$\Phi ^{[ M]}$};
		\draw (261.74,125.43) node  [font=\scriptsize,color={rgb, 255:red, 208; green, 2; blue, 27 }  ,opacity=1 ]  {$Evaluated\ at\ e^{[ M]}(\vec{x}) \ \in \ X^{M}$};
		\draw (392.21,115.64) node    {$f_{\hat{\mu }[\vec{x}]}(\vec{x})$};
		\draw (456.74,66.43) node  [font=\scriptsize,color={rgb, 255:red, 208; green, 2; blue, 27 }  ,opacity=1 ]  {$Evaluated\ at\ \vec{x} \in \ X^M$};
	\end{tikzpicture}
	\caption{Commutative diagram on the relation of canonical feature map of $k^{[M]}$ and KMEs.}
	\label{fig.ds_kme_diagram}
\end{figure}


An interesting situation arises if we consider the weak$\ast$ convergence of empirical probability measures, metrized by $\dkm$, and the convergence of their embeddings. 
Consider the setting of Propositions \ref{prop.ds_mfl} and \ref{prop.uniform_continuity_ds} and assume additionally that the double sum kernels \eqref{eq.double_sum_kernel} are uniformly continuous, so that Theorem \ref{thm.mfl_kernel} applies and we have the mean field limit kernel $k$ and its associated RKHS $H_k$, as well as convergence (of a subsequence) of $k^{[M]}$ to $k$.
Let $\vec{x}_M\in X^M$ with $\muhat[\vec{x}_M]\convmode{\dkm}\mu$ for some $\mu\in\Pb(X)$.
Each empirical measure $\muhat[\vec{x}_M]$ can be embedded into $H_0$ via the kernel mean embeddings $f^{k_0}_{\muhat[\vec{x}_M]}$ 
and into $H_M$ by first identifying it with $\vec{x}_M$ and then using the canonical feature map $\Phi^{[M]}$.
Assume now that $k_0$ is characteristic, i.e., the map $\Pb(X)\rightarrow H_k$, $\mu \mapsto f_\mu^k$ is injective.
Under this assumption, convergence of the kernel mean embeddings metrizes the weak$\ast$ topology \cite[Theorem~12]{simon2018kernel}, so we get that $f^{k_0}_{\muhat[\vec{x}_M]} \convmode{H_0} f^{k_0}_{\mu}$.
Since $k$ is the MFL of $k^{[M]}$ and the former is continuous w.r.t. $\dkm$, we also get up to a subsequence
$k^{[M]}(\cdot, \vec{x}_M) \rightarrow k(\cdot,\mu)$ as a mean field limit. Note that the kernel mean embeddings that appear here are all well-defined, cf. \cite[Theorem~1]{sriperumbudur2010hilbert}.

The preceding discussion is summarized as a diagram in Figure \ref{fig.ds_kme_mfl_diagram}.
\begin{figure}
	\centering
	\tikzset{every picture/.style={line width=0.75pt}} 
	
	\begin{tikzpicture}[x=0.65pt,y=0.75pt,yscale=-1,xscale=1]
		
		\draw    (288.5,88) -- (288.5,39) ;
		\draw [shift={(288.5,37)}, rotate = 90] [color={rgb, 255:red, 0; green, 0; blue, 0 }  ][line width=0.75]    (10.93,-3.29) .. controls (6.95,-1.4) and (3.31,-0.3) .. (0,0) .. controls (3.31,0.3) and (6.95,1.4) .. (10.93,3.29)   ;
		\draw    (352.71,100.8) -- (449.26,102.38) ;
		\draw [shift={(451.26,102.42)}, rotate = 180.94] [color={rgb, 255:red, 0; green, 0; blue, 0 }  ][line width=0.75]    (10.93,-3.29) .. controls (6.95,-1.4) and (3.31,-0.3) .. (0,0) .. controls (3.31,0.3) and (6.95,1.4) .. (10.93,3.29)   ;
		\draw    (54.33,120.21) -- (54.33,164.98) ;
		\draw [shift={(54.33,166.98)}, rotate = 270] [color={rgb, 255:red, 0; green, 0; blue, 0 }  ][line width=0.75]    (10.93,-3.29) .. controls (6.95,-1.4) and (3.31,-0.3) .. (0,0) .. controls (3.31,0.3) and (6.95,1.4) .. (10.93,3.29)   ;
		\draw    (286.47,121.07) -- (286.47,165.85) ;
		\draw [shift={(286.47,167.85)}, rotate = 270] [color={rgb, 255:red, 0; green, 0; blue, 0 }  ][line width=0.75]    (10.93,-3.29) .. controls (6.95,-1.4) and (3.31,-0.3) .. (0,0) .. controls (3.31,0.3) and (6.95,1.4) .. (10.93,3.29)   ;
		\draw    (213.4,103.54) -- (120.64,102.57) ;
		\draw [shift={(118.64,102.55)}, rotate = 0.6] [color={rgb, 255:red, 0; green, 0; blue, 0 }  ][line width=0.75]    (10.93,-3.29) .. controls (6.95,-1.4) and (3.31,-0.3) .. (0,0) .. controls (3.31,0.3) and (6.95,1.4) .. (10.93,3.29)   ;
		\draw    (218.36,184.21) -- (125.6,183.24) ;
		\draw [shift={(123.6,183.22)}, rotate = 0.6] [color={rgb, 255:red, 0; green, 0; blue, 0 }  ][line width=0.75]    (10.93,-3.29) .. controls (6.95,-1.4) and (3.31,-0.3) .. (0,0) .. controls (3.31,0.3) and (6.95,1.4) .. (10.93,3.29)   ;
		\draw    (351.71,26.8) -- (448.26,28.38) ;
		\draw [shift={(450.26,28.42)}, rotate = 180.94] [color={rgb, 255:red, 0; green, 0; blue, 0 }  ][line width=0.75]    (10.93,-3.29) .. controls (6.95,-1.4) and (3.31,-0.3) .. (0,0) .. controls (3.31,0.3) and (6.95,1.4) .. (10.93,3.29)   ;
		\draw    (489.5,88) -- (489.5,39) ;
		\draw [shift={(489.5,37)}, rotate = 90] [color={rgb, 255:red, 0; green, 0; blue, 0 }  ][line width=0.75]    (10.93,-3.29) .. controls (6.95,-1.4) and (3.31,-0.3) .. (0,0) .. controls (3.31,0.3) and (6.95,1.4) .. (10.93,3.29)   ;
		
		\draw (286.89,21.38) node    {$\mu \textcolor[rgb]{0.29,0.56,0.89}{\in }\mathcal{\textcolor[rgb]{0.29,0.56,0.89}{P}}\textcolor[rgb]{0.29,0.56,0.89}{(}\textcolor[rgb]{0.29,0.56,0.89}{\vec{x}}\textcolor[rgb]{0.29,0.56,0.89}{)} \ $};
		\draw (492.22,22.44) node    {$f_{\mu }^{k_{0}}\textcolor[rgb]{0.29,0.56,0.89}{\in H}\textcolor[rgb]{0.29,0.56,0.89}{_{0}}$};
		\draw (282.76,101.43) node    {$\textcolor[rgb]{0.29,0.56,0.89}{X}\textcolor[rgb]{0.29,0.56,0.89}{^{M}}\textcolor[rgb]{0.29,0.56,0.89}{\ \ni }\vec{x}_{M} \simeq \ \hat{\mu }[\vec{x}_{M}]$};
		\draw (394.32,14.07) node  [font=\scriptsize,color={rgb, 255:red, 208; green, 2; blue, 27 }  ,opacity=1 ]  {$KME$};
		\draw (269.61,57.44) node  [font=\scriptsize,color={rgb, 255:red, 208; green, 2; blue, 27 }  ,opacity=1 ]  {$d_{KR}$};
		\draw (491.39,102.24) node    {$f_{\hat{\mu }[\vec{x}_{M}]}^{k_{0}}\textcolor[rgb]{0.29,0.56,0.89}{\in H}\textcolor[rgb]{0.29,0.56,0.89}{_{0}}$};
		\draw (505.87,61.28) node  [font=\scriptsize,color={rgb, 255:red, 208; green, 2; blue, 27 }  ,opacity=1 ]  {$H_{0}$};
		\draw (396.79,112.09) node  [font=\scriptsize,color={rgb, 255:red, 208; green, 2; blue, 27 }  ,opacity=1 ]  {$KME$};
		\draw (58.06,101.43) node    {$\textcolor[rgb]{0.29,0.56,0.89}{H}\textcolor[rgb]{0.29,0.56,0.89}{_{M}}\textcolor[rgb]{0.29,0.56,0.89}{\ni } \ k^{[ M]}( \cdot ,\vec{x})$};
		\draw (173.78,89.54) node  [font=\scriptsize,color={rgb, 255:red, 208; green, 2; blue, 27 }  ,opacity=1 ]  {$\Phi ^{[ M]}$};
		\draw (58.88,183.84) node    {$\textcolor[rgb]{0.29,0.56,0.89}{H}\textcolor[rgb]{0.29,0.56,0.89}{_{k}}\textcolor[rgb]{0.29,0.56,0.89}{\ni } \ k( \cdot ,\mu )$};
		\draw (177.91,194.49) node  [font=\scriptsize,color={rgb, 255:red, 208; green, 2; blue, 27 }  ,opacity=1 ]  {$\Phi _{k}$};
		\draw (301.83,143.32) node  [font=\scriptsize,color={rgb, 255:red, 208; green, 2; blue, 27 }  ,opacity=1 ]  {$d_{KR}$};
		\draw (61.43,136.38) node  [font=\scriptsize,color={rgb, 255:red, 208; green, 2; blue, 27 }  ,opacity=1 ]  {$MFL\ \ \ \ M\ \rightarrow \infty $};
		\draw (284.41,180.99) node    {$\mu \textcolor[rgb]{0.29,0.56,0.89}{\in }\mathcal{\textcolor[rgb]{0.29,0.56,0.89}{P}}\textcolor[rgb]{0.29,0.56,0.89}{(}\textcolor[rgb]{0.29,0.56,0.89}{\vec{x}}\textcolor[rgb]{0.29,0.56,0.89}{)} \ $};

	\end{tikzpicture}
	\caption{Diagram illustration of the relations of double sum kernel, KME and MFL.}
	\label{fig.ds_kme_mfl_diagram}
\end{figure}

\subsection{Gaussian kernels}
As an illustration of the preceding developments, we now present a simple, concrete example. The example is a  particular case of radial basis functions. Its main property is the symmetry and is fulfilled in particular for Gaussian kernels. Those are A popular choice for a kernel in machine learning.
A Gaussian kernel $k_\gamma: \R^d \times \R^d \rightarrow \R$ is given by  $k_\gamma(x,x')=\exp\left(-\|x-x'\|^2/2\gamma \right)$, where $\gamma\in\Rp$ plays the role of a lengthscale.
For more details on this kernel and its associated RKHS, we refer to \cite[Section~4.4]{steinwart2008support}.

We start with the pullback construction: Let $K\subseteq \R^d$ be nonempty and compact and define $\phi_M: K^M \rightarrow \R$, $\phi_M(\vec x)=\frac{1}{M}\sum_{m=1}^M x_m$. 
It is then immediately clear that $k_0=k_\gamma$, restricted to $K\times K$, and $\phi_M$ fulfill the assumptions of Proposition \ref{prop.pullback} with $Y=\R^d$. { The pullback construction allows to ensure the symmetry required by Assumption \ref{assump.mfl_kernel_cond} and hence allows a mean field limit kernel and associated mean field RKHS.}

Let us turn to the double sum kernel construction: Let again $K\subseteq \R^d$ be nonempty and compact. It is clear that $k_\gamma$ fulfills the conditions of Proposition \ref{prop.ds_mfl}. Furthermore, since $k_\gamma$ is continuous, the topology induced by $d_{k_\gamma}$ is coarser than the relative topology on $K$ induced by the Euclidean distance \cite[Lemma~4.29]{steinwart2008support}, hence $(K,d_{k_\gamma})$ is separable and also Proposition \ref{prop.uniform_continuity_ds} applies.
In particular, if we replace the Euclidean distance by $d_{k_\gamma}$, the mean field limit of the double sum kernels exists in the sense of Theorem \ref{thm.mfl_kernel}, cf. Remark \ref{remark.ds_metric}

\section{Conclusion} \label{section.conclusion}
In this article, we presented appropriate conditions for sequences of kernels to exhibit a mean field limit.
We rigorously proved the existence of this limit and showed that it is a kernel, having a corresponding reproducing kernel Hilbert space.
Furthermore, we investigated this latter object and how it relates to the kernels leading to the limit. In particular, we showed the commutative relationship in this context.
Additionally, we provided two example classes of appropriate kernel sequences that are based on established concepts in the context of kernel methods. { A possible drawback of the presented method are the strong symmetry assumptions both on functionals as well as kernels. Therefore, possible applications might be limited to the approximation of large scale but symmetric functionals. 
}

\section*{Acknowledgments} 
This work is funded in part under the Excellence Strategy of the Federal Government and the Länder (G:(DE-82)EXS-SF-SFDdM035), which the authors gratefully acknowledge. The authors further thank the Deutsche Forsch\-ungs\-ge\-mein\-schaft (DFG, German Research Foundation) for the financial support through 320021702/GRK2326, 33849990/IRTG-2379, CRC1481, 442047500/SFB1481, HE5386/18-1,19-2,22-1,23-1 and under Germany’s Excellence Strategy EXC-2023 Internet of Production 390621612.

\bibliographystyle{siam} 
\bibliography{preprint} 

\medskip

\end{document}